\documentclass{article}

\usepackage{microtype}
\usepackage{graphicx}
\usepackage{subcaption}
\usepackage{booktabs} 
\usepackage{longtable}
\usepackage{amsmath}
\usepackage{amssymb}
\usepackage{bm}
\usepackage{multirow}
\usepackage{amsmath, amsthm}
\usepackage{xr}

\usepackage{hyperref}

\usepackage[preprint]{icml2026}

\usepackage{amsmath}
\usepackage{amssymb}
\usepackage{mathtools}
\usepackage{amsthm}
\usepackage{stmaryrd}

\usepackage[capitalize,noabbrev]{cleveref}

\theoremstyle{plain}
\newtheorem{theorem}{Theorem}[section]
\newtheorem{proposition}[theorem]{Proposition}

\theoremstyle{definition}

\newtheorem{assumption}[theorem]{Assumption}
\theoremstyle{remark}

\usepackage[textsize=tiny]{todonotes}

\icmltitlerunning{Submission and Formatting Instructions for ICML 2026}

\begin{document}

\twocolumn[
  \icmltitle{Taming the Instability: A Robust Second-Order Optimizer \\ for Federated Learning over Non-IID Data}



  \icmlsetsymbol{equal}{*}

  \begin{icmlauthorlist}
    \icmlauthor{Yuanqiao Zhang}{A,B,D}
    \icmlauthor{Tiantian He}{B}
    \icmlauthor{Yuan Gao}{A}
    \icmlauthor{Yixin Wang}{A}
    \icmlauthor{Yew-Soon Ong}{B,E}
    \icmlauthor{Maoguo Gong}{A}
    \icmlauthor{A.K. Qin}{F}
    \icmlauthor{Hui Li}{D}
  \end{icmlauthorlist}

  \icmlaffiliation{A}{Key Laboratory of Collaborative Intelligence Systems of Ministry of Education, Xidian University, Xi’an, Shaanxi, China}
  \icmlaffiliation{B}{Centre for Frontier AI Research, Institute of High Performance Computing, Agency for Science, Technology and Research, Singapore}
  \icmlaffiliation{D}{State Key Laboratory of ISN, School of Cyber Engineering, Xidian University, Xi'an, Shaanxi, China}
  \icmlaffiliation{E}{College of Computing and Data Science, Nanyang Technological University, Singapore}
  \icmlaffiliation{F}{Department of Computing Technologies, Swinburne
 University of Technology, Melbourne, Australia}

  \icmlcorrespondingauthor{Yuanqiao Zhang}{zhangyuanqiao@xidian.edu.cn}
  \icmlcorrespondingauthor{Tiantian He}{he\_tiantian@a-star.edu.sg}
  \icmlcorrespondingauthor{Yew-Soon Ong}{asysong@ntu.edu.sg}
  \icmlcorrespondingauthor{Maoguo Gong}{gong@ieee.org}

  \icmlkeywords{Machine Learning, ICML}

  \vskip 0.3in
]



\printAffiliationsAndNotice{}  

\begin{abstract}

In this paper, we present Federated Robust Curvature Optimization (FedRCO), a novel second-order optimization framework designed to improve convergence speed and reduce communication cost in Federated Learning systems under statistical heterogeneity. Existing second-order optimization methods are often computationally expensive and numerically unstable in distributed settings. In contrast, FedRCO addresses these
 challenges by integrating an efficient approximate curvature optimizer with a provable stability mechanism. Specifically, FedRCO incorporates
 three key components: (1) a Gradient Anomaly Monitor that detects and mitigates exploding gradients in real-time, (2) a Fail-Safe Resilience protocol that resets optimization states upon numerical instability, and (3) a Curvature-Preserving Adaptive Aggregation strategy that safely integrates global knowledge without erasing the local curvature geometry. Theoretical analysis shows that FedRCO can effectively mitigate instability and prevent unbounded updates while preserving optimization efficiency.  Extensive experiments show that FedRCO achieves superior robustness against diverse non-IID scenarios while achieving higher accuracy and faster convergence than both state-of-the-art first-order and second-order methods.

\end{abstract}

\section{Introduction}

With the rapid proliferation of Internet of Things (IoT) devices and mobile terminals, massive amounts of data are continuously generated at the network edge. Although such data are highly valuable for training machine learning models, they often contain sensitive information, rendering centralized data collection and training impractical or undesirable. To resolve the tension between data utilization and privacy preservation, Google introduced Federated Learning (FL) \cite{McMahan2016CommunicationEfficientLO}, a distributed machine learning paradigm that enables collaborative model training without requiring the sharing of raw data. In FL, clients perform local training on their private datasets and communicate only model updates—such as parameters or gradients—to a central server for aggregation, thereby ensuring data are kept confined to local devices.

Despite its growing success, FL still faces substantial optimization challenges. Mainstream FL algorithms such as FedAvg \cite{McMahan2016CommunicationEfficientLO} and its variants are based on first-order Stochastic Gradient Descent (SGD). Although computationally efficient, these methods ignore curvature information in the loss landscape, often leading to slow convergence and accordingly requiring a large number of communication rounds. This is particularly problematic in bandwidth-constrained edge environments. Moreover, the inherent statistical heterogeneity of federated data, such as non-independent and identically distributed (non-IID) distribution, causes inconsistencies among local learning objectives, leading to client drift \cite{pmlr-v119-karimireddy20a} and further degrading convergence stability or even causing divergence. These challenges call for more principled and curvature-aware optimization mechanisms.

To overcome these limitations, second-order optimization techniques have recently regained attention \cite{abdulkadirov2023survey}. By exploiting curvature information through Hessian or its approximations, these methods can adaptively scale updates and potentially achieve much faster convergence. Among them, Kronecker-Factored Approximate Curvature (K-FAC) \cite{martens2015optimizing} has attracted significant attention for its efficient computational approximation. By approximating the large Fisher matrix as a layer-wise Kronecker product, K-FAC makes second-order optimization feasible for deep neural networks.

However, our analysis shows that directly applying second-order optimization in federated learning is highly nontrivial. First, the limited computational capacity of edge devices necessitates small local batch sizes, which often lead to a rank-deficient Hessian, thereby amplifying noise and producing extensive updates. Second, under non-IID data distributions, locally estimated curvature can severely mismatch the global loss landscape, resulting in overly aggressive update steps that destabilize global optimization or even cause divergence. Therefore, a critical challenge is how to use the accelerated convergence brought by second-order optimization while taming its instability in the federated settings.

To address the aforementioned challenges, we propose Federated Robust Curvature Optimization (FedRCO), a second-order optimization framework that enables stable and communication-efficient training under heterogeneous data distributions. To the best of our knowledge, this work presents the first comprehensive study that theoretically analyzes the instability of second-order optimization in federated learning, providing both theoretical characterization and feasible solutions to address these challenges. Our main contributions are as follows.
\begin{itemize}
  \item We rigorously analyze the fundamental causes of second-order optimization instability in federated learning and demonstrate which factors lead to unbounded update norms.
  \item We develop FedRCO, a principled curvature-adaptive federated optimization framework that provably controls update magnitudes and stabilizes second-order information under ill-conditioned and heterogeneous settings.
  \item We propose a simple yet effective aggregation strategy tailored for non-IID scenarios. By safely fusing global knowledge without disrupting local geometry, our method supports robust second-order optimization with minimal computational and communication costs.

  \item Extensive experiments on standard benchmarks demonstrate that FedRCO achieves consistently faster convergence and improved final accuracy while significantly reducing communication rounds compared to both baseline first-order and state-of-the-art (SOTA) federated second-order methods.
\end{itemize}

In addition to the background, motivations, methodology, analysis, and experiments presented in \textbf{Sections 2-6}, we include an extensive theoretical and analytical treatment in the appendices. \textbf{Appendices A-D} respectively provide a detailed analysis of gradient instability in federated second-order optimization, the effectiveness of second-order optimization under federated and non-IID settings, the convergence properties of second-order FL, and the derivations of the corresponding optimization bounds. \textbf{Appendix E} provides additional experimental setups and results \footnote{Our code is available in the Supplementary Material.}.

\section{Related work}

\subsection{Federated Learning}

Federated Learning (FL) was pioneered by FedAvg \cite{McMahan2016CommunicationEfficientLO}, which minimizes a weighted global loss through periodic local SGD updates. Leaving $\boldsymbol\theta$ denote the global model parameter and $t \in \{0, ..., T-1\}$ denote the communication round, the parameter process on the client $c$ can be described as:
\begin{align}
\boldsymbol{\theta}_c^{t+1}=\boldsymbol{\theta}_c^{t}-\eta\nabla \mathcal{L}_c(\boldsymbol{\theta}_c^{t}),
\end{align}
where $\forall c \in \{1, ..., C\}$ is the index of the clients, $\eta$ is the learning rate, and $\nabla \mathcal{L}_c(\boldsymbol{\theta}_c^{t})$ denotes the gradient of the loss function $\mathcal{L}(\boldsymbol\theta)$ calculated on the local dataset. Thus, the global function $\mathcal{L}_{global}(\boldsymbol\theta)$ of FL can be denoted as:
\begin{align}
\mathcal{L}_{global}(\theta^{t+1})=\sum_{c=1}^C\frac{n_c}{n}\mathcal{L}_c(\theta_c^t),
\end{align}
where $n_c$ is the data volume of client $c$, and $n=\sum_{c=1}^Cn_c$. However, statistical heterogeneity resulting from non-IID data often leads to {client drift} \cite{pmlr-v119-karimireddy20a}, where local optima diverge from the global objective. To address this, FedProx \cite{li2020federated} introduces a proximal regularization term to limit excessive local deviations.

Leveraging curvature information in FL offers the potential for faster convergence. Early approaches, such as LocalNewton \cite{gupta2021localnewton} and FedNL \cite{safaryan2021fednl}, utilize local Newton steps or require transmitting Hessians, incurring prohibitive communication and computational overhead. Recent research, such as FedPM \cite{ishii2025fedpm}, performs server-side aggregation of covariance matrices to refine preconditioning with reduced costs.

\subsection{Natural Gradient Descent}

Natural Gradient Descent (NGD) updates parameters via
\begin{equation}
\boldsymbol {\theta}  \leftarrow \boldsymbol {\theta} - \eta \boldsymbol {F} ^{-1} \nabla \mathcal{L}(\boldsymbol {\theta} ),
\end{equation}
where $\boldsymbol{F}$ is the Fisher Information Matrix (FIM) representing the sensitivity of the model's output distribution. To bypass the $O(d^3)$ inversion bottleneck, Kronecker-Factored Approximate Curvature (K-FAC) \cite{martens2015optimizing} approximates the layer-wise FIM as a Kronecker product:
\begin{equation}
\boldsymbol{F}^{(l)} \approx \boldsymbol{\Omega}^{(l)} \otimes \boldsymbol{\Gamma}^{(l)},
\end{equation}
where $l$ denotes the layer of the model,  $\boldsymbol{\Omega}^{(l)} = \mathbb{E} \left[ \boldsymbol{A}^{(l)} (\boldsymbol{A}^{(l)})^\top \right]$ and $\boldsymbol{\Gamma}^{(l)} = \mathbb{E} \left[ \boldsymbol{G}^{(l)} (\boldsymbol{G}^{(l)})^\top \right]$ denote the covariance of the input activation $\boldsymbol{A}^{(l)}$ and the pre-activation gradient $\boldsymbol{G}^{(l)}$, respectively. Leveraging the inverse property of the Kronecker product $(\boldsymbol{\Omega} \otimes \boldsymbol{\Gamma})^{-1} = \boldsymbol{\Omega}^{-1} \otimes \boldsymbol{\Gamma}^{-1}$, K-FAC transforms the dense $d \times d$ inversion into the inversion of two much smaller matrices, significantly reducing computational complexity \cite{tang2021skfac}.

\section {Theoretical Motivations}

To date, the systematic and theoretical understanding of instability phenomena in FL, along with principled mechanisms to rectify them, remains largely unexplored. In this paper, we systematically initialize this analysis and propose principled mechanisms accordingly.

\subsection{ Gradient Instability}
Directly integrating second-order information into FL faces numerical instability. In this paper, we identify two critical failure modes: \textbf{Rank Deficiency} and \textbf{Curvature Mismatch}, and establish theoretical connections between the modes and numerical instability.
\begin{proposition}\label{prop1}
\textbf{(Rank Deficiency)} When batch size $B \ll d$, the empirical FIM $\hat{\boldsymbol{F}}_c$ is rank-deficient, allowing sampling noise in the null space to induce unbounded updates (detailed in \textbf{Appendix A.2}).
\end{proposition}
\begin{proposition}\label{prop2}
\textbf{(Curvature Mismatch)} If local curvature underestimates global steepness under non-IID data, the global quadratic penalty becomes unbounded, leading to divergence (detailed in \textbf{Appendix A.3}).
\end{proposition}

\subsection{Advantages of Second-order Optimization}
Despite these risks, second-order methods offer superior geometric robustness. While SGD suffers from slow convergence as the condition number $\kappa \to \infty$, Fisher preconditioning effectively transforms the geometry into an isotropic shape (ideally $\kappa \to 1$). The following two theorems show that adopting second-order optimization promises to dramatically accelerate the convergence, if the previously mentioned numerical instability can be effectively addressed.

\begin{proposition}\textbf{(Impact of Condition Number.)} SGD convergence is bounded by $\frac{\kappa - 1}{\kappa + 1}$ . In non-IID scenarios, $\kappa \to \infty$, causing slow convergence. (detailed in \textbf{Appendix B.1})\label{prop3}\end{proposition}

\begin{proposition}\textbf{(Affine Invariance.)} Fisher preconditioning isotropizes the parameter space, ideally achieving $\kappa \to 1$ regardless of the native landscape curvature (detailed in \textbf{Appendix B.2}).\label{prop4}
\end{proposition}

\textbf{Remark.} With \textbf{Propositions \ref{prop1} to \ref{prop4}}, we have established a theoretical foundation for identifying key issues causing the unstable federated optimization and feasible solutions, which pave the way for our approach.

\section{Methodology}

In this section, we present our primary contribution: a robust second-order optimization framework for FL. In \textbf{Section 4.1}, we first describe the client-side optimization procedure, where a block-diagonal approximation of Fisher is adopted. In \textbf{Section 4.2}, we focus on ill-conditioned curvature estimates caused by rank deficiency and statistical staleness, and introduce a unified two-stage mechanism consisting of monitoring and resilience. \textbf{Section 4.3} presents implementation details, including an aggregation strategy and a lazy inverse update strategy. All notation is described in  \textbf{Appendix A}.

\subsection{Block-Wise Curvature Estimation via K-FAC}
To make second-order optimization practical in the standard resource-limited environments \cite{t2020personalized}, where each epoch only processes a single mini-batch of data, we adopt a block-diagonal approximation of FIM. The client-side learning process can be described as follows:
\begin{align}
\boldsymbol{\theta}_c^{t}=\boldsymbol{\theta}_c^{t}-\eta \boldsymbol F^{-1} \nabla \mathcal{L}_c(\boldsymbol{\theta}_c^{t}) .
\end{align}
For a neural network layer, taking a fully connected layer $l$ as an example, the corresponding FIM $\boldsymbol{F}^{(l)}$ is approximated as the Kronecker product of smaller matrices $\boldsymbol{F}^{(l)} \approx \boldsymbol{\Omega}^{(l)} \otimes \boldsymbol{\Gamma}^{(l)}$. Since computing the natural-gradient update requires the inverse of $\boldsymbol{F}$, we leverage a key property of the Kronecker product $(\boldsymbol{\Omega} \otimes \boldsymbol{\Gamma})^{-1} = \boldsymbol{\Omega}^{-1} \otimes \boldsymbol{\Gamma}^{-1}$, which enables the two factors to be inverted independently. To further ensure numerical stability, we apply Tikhonov damping \cite{martens2015optimizing} with parameter $\epsilon$ together with the $\pi$-correction method to balance the relative scales of the two Kronecker factors. Thus, the empirical statistics of the Kronecker factors can be defined as:
\begin{align}
&[\widehat{\boldsymbol{\Omega}}^{(l)}]^{-1} = (\boldsymbol{\Omega}_t^{(l)} + \pi \sqrt{\epsilon} \boldsymbol{I})^{-1},\notag \\
&[\widehat{\boldsymbol{\Gamma}}^{(l)}]^{-1} = (\boldsymbol{\Gamma}_t^{(l)} + \frac{1}{\pi} \sqrt{\epsilon} \boldsymbol{I})^{-1},
\end{align}
where $\pi = \sqrt{\text{tr}(\boldsymbol{A}) / \text{tr}(\boldsymbol{G})}$.

For convolutional layers, we employ the unfold operation to reshape spatial features into two-dimensional matrices \cite{grosse2016kronecker}, enabling consistent covariance estimation across layers. Under this formulation, the Kronecker factors are defined as:
$\boldsymbol{\Omega}^{(l)} = \mathbb{E} \left[ \llbracket\boldsymbol{A}^{(l)} \rrbracket \llbracket\boldsymbol{A}^{(l)}\rrbracket^\top \right]$,
$\boldsymbol{\Gamma}^{(l)}=\mathbb{E}\left[\boldsymbol{G}^{(l)}\boldsymbol{G}^{(l)\top}\right]$,
where $\llbracket \cdot \rrbracket$ denotes the operation that extracts local patches around each spatial location, stretches them into vectors, and stacks the resulting vectors into a matrix. Subsequent empirical covariance statistics $\widehat{\boldsymbol{\Omega}}$ and $\widehat{\boldsymbol{\Gamma}}$ are computed in the same way as for fully connected layers.

Since only a single mini-batch is processed per local epoch, we employ exponential moving averages (EMA) to mitigate the stochastic noise inherent in small-batch training. Specifically, at the local epoch $k$, the updates are:
\begin{align}
\hat{\boldsymbol{A}}_k^{(l)} = \alpha \mathbb{E}[\boldsymbol{A}^{(l-1)} \boldsymbol{A}^{(l-1)\top}] + (1-\alpha) \boldsymbol{A}_{k-1}^{(l)} ,\notag \\
\hat{\boldsymbol{G}}_k^{(l)} = \alpha \mathbb{E}[\boldsymbol{G}^{(l-1)} \boldsymbol{G}^{(l-1)\top}] + (1-\alpha) \boldsymbol{G}_{k-1}^{(l)},
\end{align}
where $\alpha$ represents the hyperparameter to control the average ratio.

We substitute EMA to obtain $\widehat{\boldsymbol{\Omega}}$ and $\widehat{\boldsymbol{\Gamma}}$. Finally, the preconditioned natural gradient update for client $c$ is:
\begin{align}
&\text{vec}(\tilde{\nabla}_{\boldsymbol{\theta}}^{(l)})
=\widehat{\boldsymbol{\Omega}}^{(l)-1} \nabla \mathcal{L}_c(\boldsymbol{\theta}_c) \widehat{\boldsymbol{\Gamma}}^{(l)-1} \notag \\
&= (\boldsymbol{\Omega}^{(l)} + \pi \sqrt{\epsilon} \boldsymbol{I})^{-1} \nabla \mathcal{L}_c(\boldsymbol{\theta}_c)  (\boldsymbol{\Gamma}^{(l)} + \frac{\sqrt{\epsilon}}{\pi}  \boldsymbol{I})^{-1}.
\label{3.1}
\end{align}

\subsection{Gradient Anomaly Detection and Resilience}

One of the core contributions of our method is the ability to handle ill-conditioned curvature arising in non-IID federated settings which is claimed in \textbf{Section 2.3} and proven in \textbf{Appendix A}. To enable FedRCO to achieve this, we propose a two stage mechanism that stabilizes gradient updates while maintaining training efficiency.

\emph{\textbf{1) Look-Ahead Gradient Monitor}}

First, we employ a real-time monitoring mechanism to inspect the optimization trajectory. Specifically, we maintain a sliding-window history of preconditioned gradient norms,
$\mathcal{H} = \Big\{  \|\tilde{\boldsymbol{g}}_{K-k}\|, \dots, \|\tilde{\boldsymbol{g}}_{K-1}\|  \Big\}$, where $\|\cdot\|$ is the  $\ell_2$-norm, and $\tilde{\boldsymbol{g}}_{k}$ is the preconditioned gradient:
\begin{align}
    \|\tilde{\boldsymbol{g}}_k^{(l)}\| &\triangleq \sqrt{\sum_{i=1}^{d_{out}} \sum_{j=1}^{d_{in}} (\tilde{g}_{k, i, j}^{(l)})^2},\notag \\
   \|\tilde{\boldsymbol{g}}_{k}\|&=\sum_{l} \|\tilde{\boldsymbol{g}}_k^{(l)}\|,
\end{align}
where the parameters of a layer are typically represented as a matrix $\boldsymbol{\theta}^{(l)} \in \mathbb{R}^{d_{out} \times d_{in}}$. At epoch $k$, before applying the update, we calculate the anomaly score $S_k$:
\begin{equation}
S_k = \frac{\|\tilde{\nabla} \theta^{k}\|}{\text{mean}(\mathcal{H}) + \xi},
\end{equation}
where, $\text{mean}(\mathcal{H}) = \frac{1}{K} \sum_{k=1}^{K} ||\tilde{\boldsymbol{g}}_{K-k}||_2$ serves as the baseline of expected update magnitude, and $\xi$ is a small constant  to prevent division by zero. The score $S_k$ thus represents the relative divergence of the current step compared to recent steps.

Based on the anomaly score $S_k$, we categorize potential failures into two types. When $S_k > \tau_{\text{low}} \approx 10$, this typically indicates a structural mismatch between the local curvature approximation and the underlying loss landscape. Such cases are classified as {\textbf{accumulated divergence}}, which manifests as a gradual drift in the optimization trajectory. When $S_k > \tau_{\text{high}} \approx 1000$, the FIM becomes nearly singular, causing its inverse to excessively amplify gradients in arbitrary directions. These events are classified as {\textbf{sudden explosion}}, leading to impulsive noise in the updates.

\emph{\textbf{2) Robust Resilience Protocol}}

Upon detecting an anomaly, the client triggers a resilience protocol to safeguard both the local and global models. For cases of accumulated divergence, where $\tau_{\text{low}} < S_k < \tau_{\text{high}}$, the instability arises from a gradual accumulation of curvature mismatch, leading to steadily increasing gradients. In this regime, we adopt a soft rollback strategy to stabilize the update:

\begin{equation}
\boldsymbol{\theta}_{k+1} =
\begin{cases}
\boldsymbol{\theta}_k - \eta \tilde{\boldsymbol{\nabla}}_k, & \text{if } S_k \leq \tau_{\text{low}} \\
\boldsymbol{\theta}_k - \eta \tilde{\boldsymbol{\nabla}}_{stable}, & \text{if } \tau_{\text{low}}<S_k <  \tau_{\text{high}}
\end{cases}
\end{equation}
where $\tilde{\boldsymbol{\nabla}}_{stable}$ is a stable gradient upper bound.

If the anomaly score exceeds the upper threshold $S_t >  \tau_{\text{high}}$, or remains above $S_t >  \tau_{\text{low}}$ for multiple local epochs, the gradient is deemed to have caused an irreversible disruption to the optimization trajectory. In this case, a hard reset is performed. Specifically, all accumulated curvature statistics ($\boldsymbol{\Omega}, \boldsymbol{\Gamma}$) and their corresponding inverses are discarded, the local model parameters are reset to the current global model, and the hyperparameters of both the optimizer and the preconditioner are re-initialized. This step is crucial, as a gradient explosion indicates that the current curvature approximation is no longer geometrically valid.

\subsection{Implementation Details}
\emph{\textbf{1) Curvature-Preserving Adaptive Aggregation}}

Standard aggregation methods pose a risk of curvature forgetting in second-order optimization, where forcibly overwriting local parameters with a global average disrupts the delicate geometric structure encoded by the local FIM. To mitigate this in non-IID settings without incurring extra overhead, we introduce a lightweight interpolation strategy that preserves local curvature stability. Instead of a direct overwrite, the local model is updated as:
\begin{equation}
\boldsymbol{\theta}_{local}^{new} =
\begin{cases}
\gamma \cdot \boldsymbol{\theta}_{global} + (1-\gamma) \cdot \boldsymbol{\theta}_{local}^{old}, & \text{if } Acc' > Acc \\
\boldsymbol{\theta}_{global}, & \text{else}
\end{cases}
\end{equation}
where $Acc'$ and $Acc$ denote the local accuracy and the global accuracy, respectively, and $\gamma$ is dynamically adjusted based on the confidence in the local representation:
\begin{equation}
\gamma = \frac{Acc'}{Acc'+Acc}.
\end{equation}
This strategy serves as a stability filter for second-order updates. By prioritizing the local model when it significantly outperforms the global average, we prevent the curvature mismatch problem (\textbf{Proposition \ref{prop2}}) where global drift erases the precise local directions computed by K-FAC. This ensures that the aggressive second-order steps remain valid throughout the training process.

\emph{\textbf{2) Lazy Inverse Update}}

To minimize computational overhead, the inverse matrices $[\widehat{\boldsymbol{\Omega}}^{(l)}]^{-1}$ and $[\widehat{\boldsymbol{\Gamma}}^{(l)}]^{-1}$ are not recomputed at every epoch. Instead, we define an inversion interval $T_{inv}$. The curvature factors $\boldsymbol{\Omega}$ and $\boldsymbol{\Gamma}$ are accumulated continuously, but inversion occurs sparsely. This update strategy actually acts as a regularizer against mini-batch noise. The discussion of the inversion update frequency is given in \textbf{Appendix A.4}.

\section{Analysis}
We present a rigorous theoretical analysis of FedRCO, establishing convergence guarantees that effectively bound client drift under non-IID settings. Additionally, we show that our decoupled approximation and lazy updates achieve low computational complexity and communication efficiency, ensuring edge feasibility.

\subsection{Convergence Analysis and Optimization Bounds}

We present four theorems. Here, we first guarantee the convergence behavior of FedRCO from two perspectives: the local descent property on the client side and the convergence of the aggregated model on the server side.

\begin{theorem}
    \textbf{(Local Descent with Preconditioning)} For a learning rate $\eta$ satisfying $\eta \le \frac{\lambda_{min}}{L \lambda_{max}^2}$, the expected decrease in the local objective $\mathcal{L}_c$ for a single K-FAC step is lower bounded by (detailed in \textbf{Appendix C.2}):
\begin{align}
    \mathbb{E}[\mathcal{L}_c(\boldsymbol\theta^{t+1}_c)]   &\le
    \mathcal{L}_c(\boldsymbol\theta_c^t)\notag\\& - \eta \frac{\lambda_{min}}{2} \|\nabla \mathcal{L}_c(\boldsymbol\theta_c^t)\|^2 + \eta^2 \frac{L \lambda_{max}^2 \sigma^2}{2}.
\end{align}
\label{theo5.1}
\end{theorem}
\begin{theorem}
    \textbf{(Server Convergence Rate)} Second-order federated optimization converges to a neighborhood of the optimal solution $\boldsymbol\theta^*$. Specifically, for the global model ${\boldsymbol \theta}^{t}$, the error bound satisfies (detailed in \textbf{Appendix C.3}):
    \begin{align}
\mathbb{E} \|{\boldsymbol\theta}^{t+1} - \boldsymbol\theta^*\|^2 \le (1 - \rho) \mathbb{E} \|{\boldsymbol\theta}^t -\boldsymbol \theta^*\|^2 + E.
    \end{align}
    \label{theo5.2}
\end{theorem}
Afterwards, we analyze the optimization bounds of the local drift and the final global convergence bound.

\begin{theorem}
    \textbf{(Client Drift Bound)} The expected squared norm of the client drift after $K$ local steps is bounded by (detailed in \textbf{Appendix D.2}):
\begin{align}
    {e}_{drift} = \mathbb{E} \left[ \left\| \boldsymbol\theta_{c, K}^t - \boldsymbol\theta^t \right\|^2 \right] \le 2 K^2 \eta^2 \lambda_{max}^2 (\sigma^2+M^2).
\end{align}
\label{theo.d1}
\end{theorem}
\begin{theorem}
 After $T'$ communication rounds, the convergence of our method satisfies (detailed in \textbf{Appendix D.4}):
\begin{align}
   &\mathbb{E}[\mathcal{L}(\boldsymbol\theta^{T'})  - \mathcal{L}^*] \le\notag\\
   &{(1 - \rho)^{T'} (\mathcal{L}(\boldsymbol\theta^{0})  - \mathcal{L}^*)} +
   {\frac{L K \eta \lambda_{max}^2 (\sigma^2 + M^2)}{2 \mu \lambda_{min}}}.
\end{align}
\label{theo5.4}
\end{theorem}
 \textbf{Remark. Theorems \ref{theo5.1}} to \textbf{\ref{theo5.4}}  demonstrate that FedRCO achieves convergence to a neighborhood of the optimal solution while effectively bounding the client drift caused by data heterogeneity. Together with the remarkable empirical results in \textbf{Section 6}, FedRCO is shown to be a novel and theoretically grounded approach to optimizing FL systems.

\subsection{Computational Time Complexity}
We address the computational overhead by comparing FedRCO against standard first-order methods and exact second-order approaches. While exact methods incur a prohibitive $O(d^3)$ inversion cost, FedRCO decouples this burden via Kronecker factorization and lazy updates to $({d_{in}^3 + d_{out}^3})$. By amortizing the periodic inversion cost over $T_{\text{inv}}$ steps, the per-step complexity is reduced to:
\begin{equation}
O(d_{\text{in}} d_{\text{out}}(d_{\text{in}} + d_{\text{out}}) + \frac{d_{in}^3 + d_{out}^3}{T_{inv}}),
\end{equation}
where \(d_{\text{in}}\), \(d_{\text{out}}\), and \(d\)  denote the input dimensions, output dimensions, and the total number of model parameters, $ O(d_{\text{in}}  d_{\text{out}}(d_{\text{in}} + d_{\text{out}}))$ is the cost of matrix-matrix multiplications described in Eq. \ref{3.1}.

As $T_{\text{inv}}$ increases, the additional term becomes marginal compared to the $O(d_{\text{in}} d_{\text{out}})$ complexity of SGD. Coupled with the linear $O(d)$ cost of the gradient monitor, FedRCO enables efficient second-order optimization on edge devices with minimal latency overhead.

\subsection{Communication Time Complexity}
Unlike prior second-order methods that transmit heavy Hessian or covariance matrices, FedRCO maintains the same communication pattern as FedAvg, transmitting only the model parameter $\boldsymbol{\theta} \in \mathbb{R}^d$ and a scalar accuracy. This results in a minimal per-round complexity of $O(d+1) \approx O(d)$. By leveraging curvature information to rectify local updates, FedRCO significantly accelerates convergence, thereby reducing the total number of communication rounds.

\section{Experiments}

\begin{table*}[h]
\caption{The experimental results on the Dirichlet-non, Pathological-non, and IID settings with client number 100, and the party ratio 0.8.}
\setlength{\tabcolsep}{1.6mm}
	\renewcommand{\arraystretch}{1}
\begin{tabular}{c|cccccc|cccccc}
\midrule[1pt]
            & \multicolumn{6}{c|}{\textbf{CIFAR-10}}                                                                                                                                                                                  & \multicolumn{6}{c}{\textbf{EMNIST}}                                                                                                                                                                                    \\ \midrule[1pt]
\multirow{2}{*}{\textbf{Method}} & \multicolumn{3}{c|}{\textbf{Dirichlet}}                                                                         & \multicolumn{2}{c|}{\textbf{Pathological}}                                & \textbf{IID}              & \multicolumn{3}{c|}{\textbf{Dirichlet}}                                                                         & \multicolumn{2}{c|}{\textbf{Pathological}}                                & \textbf{IID}              \\ \cline{2-13}
                           & \multicolumn{1}{c|}{\textbf{$\alpha=0.1$}}   & \multicolumn{1}{c|}{\textbf{$\alpha=0.5$}}   & \multicolumn{1}{c|}{\textbf{$\alpha=1$}}     & \multicolumn{1}{c|}{{2}}     & \multicolumn{1}{c|}{{5}}     & \textbf{\textbackslash{}} & \multicolumn{1}{c|}{\textbf{$\alpha=0.1$}}   & \multicolumn{1}{c|}{\textbf{$\alpha=0.5$}}   & \multicolumn{1}{c|}{\textbf{$\alpha=1$}}     & \multicolumn{1}{c|}{{10}}    & \multicolumn{1}{c|}{{30}}    & \textbf{\textbackslash{}} \\ \midrule[1pt]
\textbf{FedAvg}            & \multicolumn{1}{c|}{0.563}          & \multicolumn{1}{c|}{0.620}          & \multicolumn{1}{c|}{0.632}          & \multicolumn{1}{c|}{0.537}          & \multicolumn{1}{c|}{0.601}          & 0.650                     & \multicolumn{1}{c|}{0.818}          & \multicolumn{1}{c|}{0.835}          & \multicolumn{1}{c|}{0.836}          & \multicolumn{1}{c|}{0.809}          & \multicolumn{1}{c|}{0.826}          & 0.838                     \\ \hline
\textbf{FedProx}           & \multicolumn{1}{c|}{0.553}          & \multicolumn{1}{c|}{0.628}          & \multicolumn{1}{c|}{0.637}          & \multicolumn{1}{c|}{0.530}          & \multicolumn{1}{c|}{0.616}          & 0.631                     & \multicolumn{1}{c|}{0.817}          & \multicolumn{1}{c|}{0.831}          & \multicolumn{1}{c|}{0.833}          & \multicolumn{1}{c|}{0.801}          & \multicolumn{1}{c|}{0.822}          & 0.839                     \\ \hline
\textbf{FedAdam}           & \multicolumn{1}{c|}{0.557}          & \multicolumn{1}{c|}{0.639}          & \multicolumn{1}{c|}{0.610}          & \multicolumn{1}{c|}{0.439}          & \multicolumn{1}{c|}{0.623}          & 0.619                     & \multicolumn{1}{c|}{0.845}          & \multicolumn{1}{c|}{0.856}          & \multicolumn{1}{c|}{0.859}          & \multicolumn{1}{c|}{0.804}          & \multicolumn{1}{c|}{0.837}          & 0.864                     \\ \hline
\textbf{FedAvgM}           & \multicolumn{1}{c|}{0.556}          & \multicolumn{1}{c|}{0.614}          & \multicolumn{1}{c|}{0.635}          & \multicolumn{1}{c|}{0.530}          & \multicolumn{1}{c|}{0.607}          & 0.658                     & \multicolumn{1}{c|}{0.814}          & \multicolumn{1}{c|}{0.832}          & \multicolumn{1}{c|}{0.832}          & \multicolumn{1}{c|}{0.800}          & \multicolumn{1}{c|}{0.823}          & 0.836                     \\ \hline
\textbf{LocalNewton}       & \multicolumn{1}{c|}{0.509}          & \multicolumn{1}{c|}{0.613}          & \multicolumn{1}{c|}{0.621}          & \multicolumn{1}{c|}{0.492}          & \multicolumn{1}{c|}{0.598}          & 0.626                     & \multicolumn{1}{c|}{0.804}          & \multicolumn{1}{c|}{0.818}          & \multicolumn{1}{c|}{0.817}          & \multicolumn{1}{c|}{0.800}          & \multicolumn{1}{c|}{0.809}          & 0.819                     \\ \hline
\textbf{FedPM}             & \multicolumn{1}{c|}{0.547}          & \multicolumn{1}{c|}{0.602}          & \multicolumn{1}{c|}{0.620}          & \multicolumn{1}{c|}{0.518}          & \multicolumn{1}{c|}{0.603}          & 0.635                     & \multicolumn{1}{c|}{0.805}          & \multicolumn{1}{c|}{0.816}          & \multicolumn{1}{c|}{0.820}          & \multicolumn{1}{c|}{0.794}          & \multicolumn{1}{c|}{0.814}          & 0.816                     \\ \midrule[1pt]
\textbf{FedRCO-ori}        & \multicolumn{1}{c|}{0.695}          & \multicolumn{1}{c|}{0.713}          & \multicolumn{1}{c|}{0.727}          & \multicolumn{1}{c|}{0.612}          & \multicolumn{1}{c|}{0.712}          & \textbf{0.725}            & \multicolumn{1}{c|}{0.854}          & \multicolumn{1}{c|}{0.869}          & \multicolumn{1}{c|}{0.871}          & \multicolumn{1}{c|}{0.851}          & \multicolumn{1}{c|}{0.859}          & \textbf{0.871}            \\ \hline
\textbf{FedRCO}            & \multicolumn{1}{c|}{\textbf{0.788}} & \multicolumn{1}{c|}{\textbf{0.742}} & \multicolumn{1}{c|}{\textbf{0.730}} & \multicolumn{1}{c|}{\textbf{0.753}} & \multicolumn{1}{c|}{\textbf{0.751}} & 0.719                     & \multicolumn{1}{c|}{\textbf{0.902}} & \multicolumn{1}{c|}{\textbf{0.886}} & \multicolumn{1}{c|}{\textbf{0.881}} & \multicolumn{1}{c|}{\textbf{0.882}} & \multicolumn{1}{c|}{\textbf{0.890}} & 0.870                     \\ \midrule[1pt]
\end{tabular}
  \label{table1}
\end{table*}

\subsection{Experimental Setup}
\textbf{Datasets and Baselines.}
We conduct extensive experiments on two widely used federated learning benchmarks: CIFAR-10 \cite{krizhevsky2009learning} and EMNIST \cite{cohen2017emnist}. To simulate non-IID data distributions, we choose Dirichlet \cite{yurochkin2019bayesian}, Pathological \cite{li2022federated}, and IID partitioning (Fig. \ref{distri} for details). Our realistic local training setup involves multiple epochs per round, with each epoch processing a single mini-batch of data. We evaluate FedRCO across Dirichlet ($Dir(\alpha), \alpha \in \{0.1, 0.5, 1\}$), Pathological $\{2,5\}$ and $\{10,30\}$ labels for CIFAR-10 and EMNIST, and IID to simulate diverse conditions. The experimental framework encompasses $ \{10, 50, 100\}$ clients with participation ratios in $\{0.1, 0.5, 0.8, 1\}$, 1600 communication rounds, and 20 local epochs per round. We compare our method against several first-order baseline methods and second-order SOTA methods: FedAvg \cite{McMahan2016CommunicationEfficientLO}, FedAvgM \cite{hsu2019measuring}, FedProx \cite{li2020federated}, FedAdam \cite{reddi2020adaptive}, LocalNewton \cite{gupta2021localnewton}, and FedPM \cite{ishii2025fedpm}. Additional details about the model architecture and implementation are provided in \textbf{Appendix E}.

\begin{figure*}[h]
  \centering
  \includegraphics[width=6.9in]{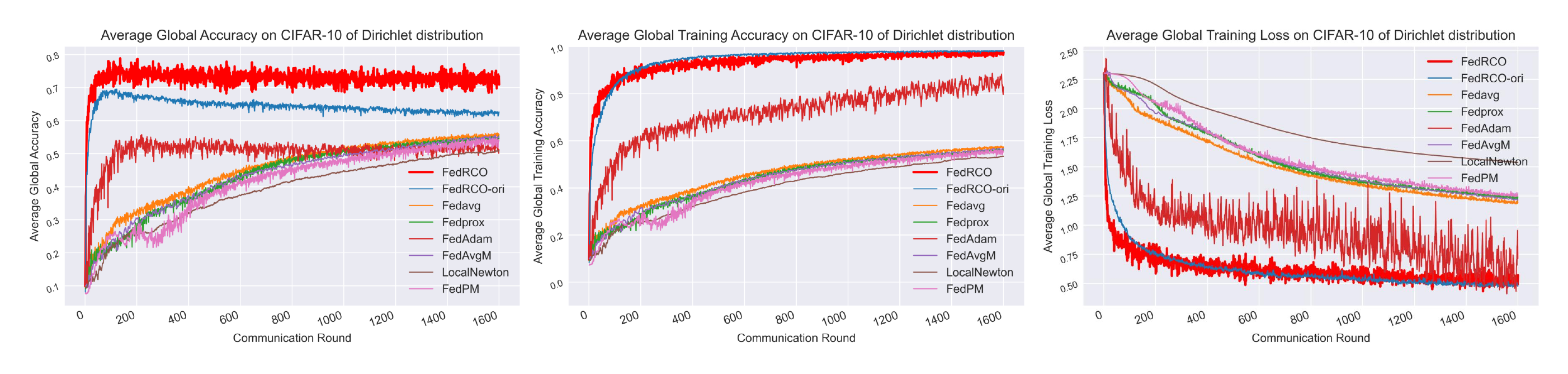}
  \caption{The experimental results averaged across all participating clients on test accuracy, training accuracy, and training loss on CIFAR-10 with $Dir(\alpha)=0.1$, client number 100, party ratio 0.8.}
  \label{4.1}
\end{figure*}

\subsection{Experimental Results}

As shown in Table \ref{table1}, FedRCO establishes a better performance lead across all datasets and settings, especially in extreme non-IID scenarios. FedRCO-ori refers to our method utilizing simple averaging instead of the proposed aggregation strategy. In the highly challenging Dirichlet $\alpha=0.1$ configuration on CIFAR-10, FedRCO achieves a $78.8\%$ accuracy, demonstrating significant improvement over standard FedAvg and overcoming robust baselines by an even wider margin. Unlike naive second-order methods that frequently struggle with convergence under high heterogeneity, FedRCO effectively harnesses curvature information to accelerate training without sacrificing stability. On the EMNIST dataset, our method maintains a decisive lead with $90.2\%$ accuracy under $\alpha=0.1$, conclusively validating its ability to neutralize local drift and handle severe statistical heterogeneity where first-order baselines falter.

The visualization of learning curves (Fig. \ref{4.1}) further confirms that FedRCO significantly improves both convergence speed and optimization stability. In scenarios with extreme data fragmentation, FedRCO quickly achieves peak accuracy, while first-order methods stagnate.  While adaptive methods like FedAdam exhibit erratic loss trajectories and LocalNewton suffers from instability, FedRCO maintains a smooth and stable decreasing training loss, empirically proving that our robust gradient monitor successfully filters out the noise inherent in distributed second-order optimization. Furthermore, the notable performance improvement between FedRCO and FedRCO-ori highlights that our specialized aggregation strategy successfully preserves client-specific representations while enhancing global generalization, ensuring that FedRCO is not merely fitting local noise but learning a fundamentally superior global model.

\begin{figure}[]
  \centering
  \includegraphics[width=3.in]{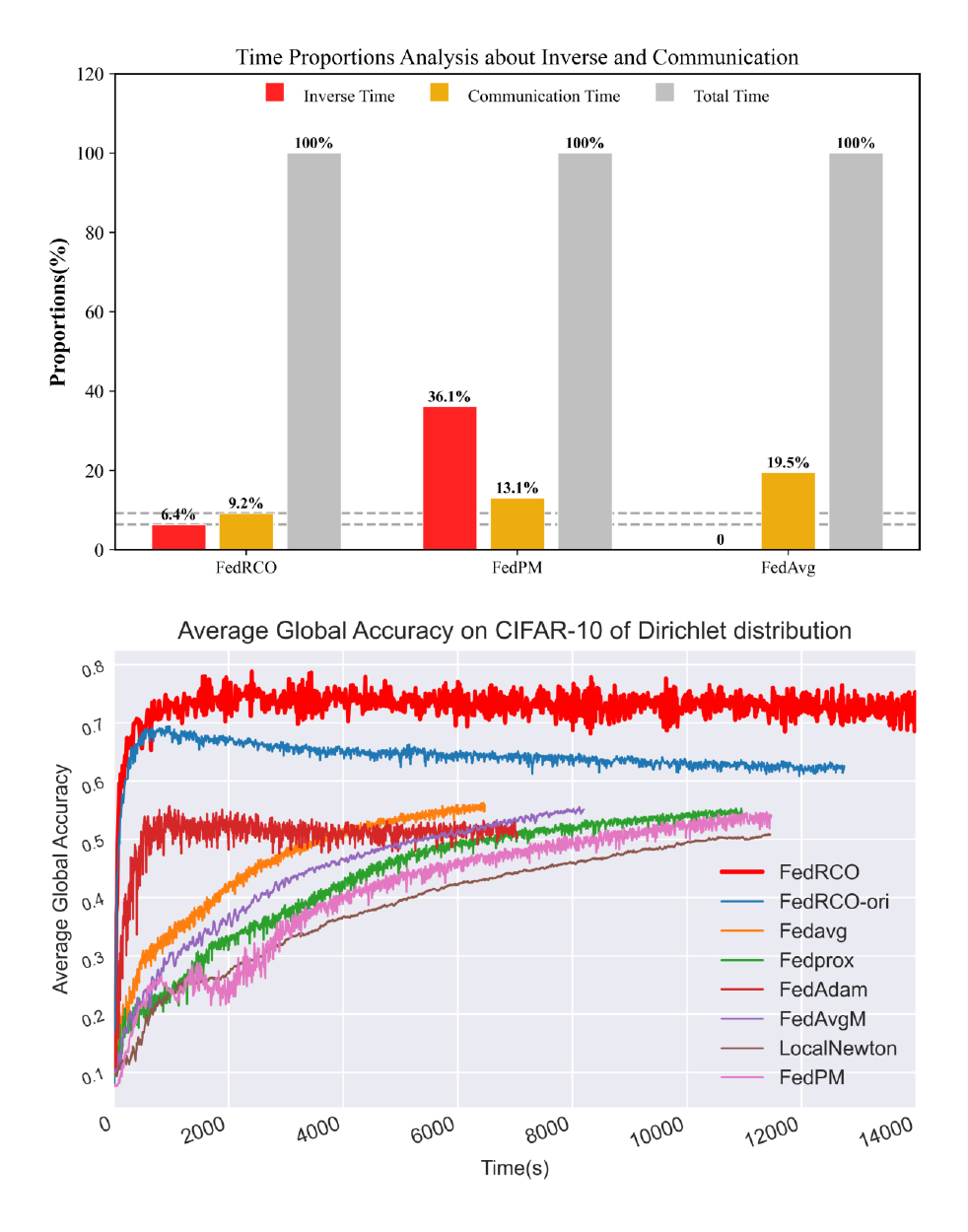}
  \caption{(a), The proportion of time spent on Matrix Inversion and Communication relative to the total training time. (b) The experimental results on average global accuracy versus wall-clock time. The horizontal axis represents time. }
  \label{4.2}
\end{figure}

\begin{figure}[!t]
  \centering
  \includegraphics[width=3in]{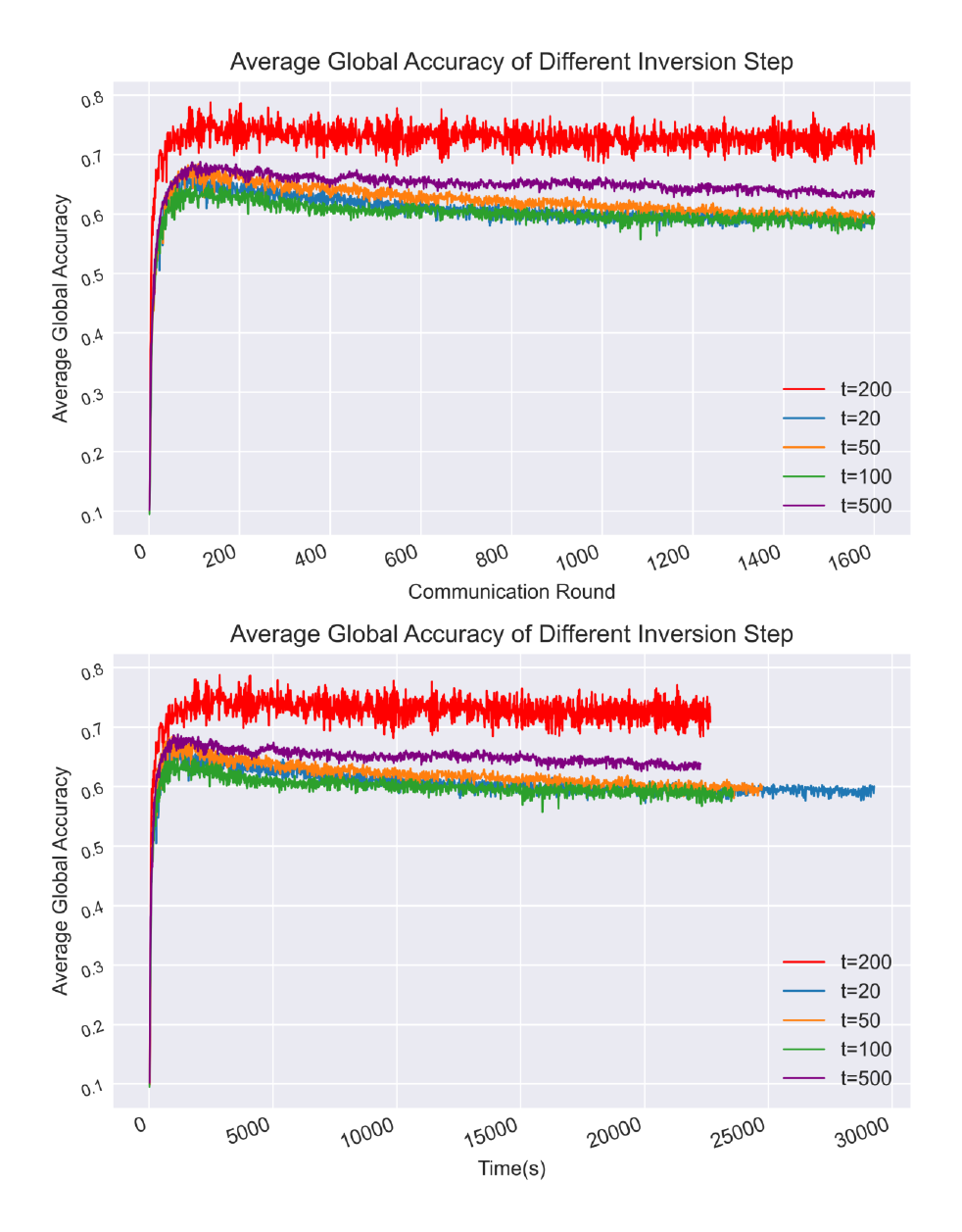}
  \caption{Impact results of Inversion Frequency ($T_{inv}$) measured by communication round and time.}
  \label{4.4}
\end{figure}

\begin{table*}[h]
\caption{The ablation experimental results about client number and participation ratio on the Dirichlet distribution with $Dir(\alpha)=0.1$. For the client number part, the party ratio is set as 0.8, and for the party ratio part, the client number is set as 100.}
\setlength{\tabcolsep}{2.15mm}
	\renewcommand{\arraystretch}{1}
\begin{tabular}{c|cccccc|cccccc}
\midrule[1pt]
           & \multicolumn{6}{c|}{\textbf{CIFAR-10}}                                                                                                                                                                       & \multicolumn{6}{c}{\textbf{EMNIST}}                                                                                                                                                                         \\ \midrule[1pt]
\multirow{2}{*}{\textbf{Method}} & \multicolumn{3}{c|}{\textbf{Client Number}}                                                                     & \multicolumn{3}{c|}{\textbf{Party Ratio}}                                                  & \multicolumn{3}{c|}{\textbf{Client Number}}                                                                     & \multicolumn{3}{c}{\textbf{Party Ratio}}                                                  \\ \cline{2-13}
                           & \multicolumn{1}{c|}{{10}}    & \multicolumn{1}{c|}{{50}}    & \multicolumn{1}{c|}{{100}}   & \multicolumn{1}{c|}{{0.1}}   & \multicolumn{1}{c|}{{0.5}}   & {1}     & \multicolumn{1}{c|}{{10}}    & \multicolumn{1}{c|}{{50}}    & \multicolumn{1}{c|}{{100}}   & \multicolumn{1}{c|}{{0.1}}   & \multicolumn{1}{c|}{{0.5}}   & {1}     \\ \midrule[1pt]
\textbf{FedAvg}            & \multicolumn{1}{c|}{0.539}          & \multicolumn{1}{c|}{0.521}          & \multicolumn{1}{c|}{0.563}          & \multicolumn{1}{c|}{0.538}          & \multicolumn{1}{c|}{0.576}          & 0.565          & \multicolumn{1}{c|}{0.779}          & \multicolumn{1}{c|}{0.814}          & \multicolumn{1}{c|}{0.818}          & \multicolumn{1}{c|}{0.806}          & \multicolumn{1}{c|}{0.814}          & 0.815          \\ \hline
\textbf{FedProx}           & \multicolumn{1}{c|}{0.557}          & \multicolumn{1}{c|}{0.573}          & \multicolumn{1}{c|}{0.553}          & \multicolumn{1}{c|}{0.540}          & \multicolumn{1}{c|}{0.575}          & 0.554          & \multicolumn{1}{c|}{0.781}          & \multicolumn{1}{c|}{0.810}          & \multicolumn{1}{c|}{0.817}          & \multicolumn{1}{c|}{0.807}          & \multicolumn{1}{c|}{0.819}          & 0.815          \\ \hline
\textbf{FedAdam}           & \multicolumn{1}{c|}{0.456}          & \multicolumn{1}{c|}{0.539}          & \multicolumn{1}{c|}{0.557}          & \multicolumn{1}{c|}{0.397}          & \multicolumn{1}{c|}{0.549}          & 0.525          & \multicolumn{1}{c|}{0.795}          & \multicolumn{1}{c|}{0.834}          & \multicolumn{1}{c|}{0.845}          & \multicolumn{1}{c|}{0.840}          & \multicolumn{1}{c|}{0.847}          & 0.846          \\ \hline
\textbf{FedAvgM}           & \multicolumn{1}{c|}{0.517}          & \multicolumn{1}{c|}{0.570}          & \multicolumn{1}{c|}{0.556}          & \multicolumn{1}{c|}{0.555}          & \multicolumn{1}{c|}{0.559}          & 0.551          & \multicolumn{1}{c|}{0.781}          & \multicolumn{1}{c|}{0.816}          & \multicolumn{1}{c|}{0.814}          & \multicolumn{1}{c|}{0.811}          & \multicolumn{1}{c|}{0.819}          & 0.812          \\ \hline
\textbf{LocalNewton}       & \multicolumn{1}{c|}{0.495}          & \multicolumn{1}{c|}{0.559}          & \multicolumn{1}{c|}{0.509}          & \multicolumn{1}{c|}{0.512}          & \multicolumn{1}{c|}{0.534}          & 0.538          & \multicolumn{1}{c|}{0.758}          & \multicolumn{1}{c|}{0.791}          & \multicolumn{1}{c|}{0.804}          & \multicolumn{1}{c|}{0.798}          & \multicolumn{1}{c|}{0.806}          & 0.802          \\ \hline
\textbf{FedPM}             & \multicolumn{1}{c|}{0.539}          & \multicolumn{1}{c|}{0.507}          & \multicolumn{1}{c|}{0.547}          & \multicolumn{1}{c|}{0.546}          & \multicolumn{1}{c|}{0.551}          & 0.543          & \multicolumn{1}{c|}{0.760}          & \multicolumn{1}{c|}{0.795}          & \multicolumn{1}{c|}{0.805}          & \multicolumn{1}{c|}{0.802}          & \multicolumn{1}{c|}{0.799}          & 0.804          \\ \midrule[1pt]
\textbf{FedRCO-ori}        & \multicolumn{1}{c|}{0.622}          & \multicolumn{1}{c|}{0.609}          & \multicolumn{1}{c|}{0.695}          & \multicolumn{1}{c|}{\textbf{0.650}} & \multicolumn{1}{c|}{0.675}          & 0.662          & \multicolumn{1}{c|}{0.805}          & \multicolumn{1}{c|}{0.839}          & \multicolumn{1}{c|}{0.854}          & \multicolumn{1}{c|}{0.853}          & \multicolumn{1}{c|}{0.853}          & 0.853          \\ \hline
\textbf{FedRCO}            & \multicolumn{1}{c|}{\textbf{0.743}} & \multicolumn{1}{c|}{\textbf{0.786}} & \multicolumn{1}{c|}{\textbf{0.788}} & \multicolumn{1}{c|}{0.630}          & \multicolumn{1}{c|}{\textbf{0.737}} & \textbf{0.764} & \multicolumn{1}{c|}{\textbf{0.885}} & \multicolumn{1}{c|}{\textbf{0.908}} & \multicolumn{1}{c|}{\textbf{0.902}} & \multicolumn{1}{c|}{\textbf{0.856}} & \multicolumn{1}{c|}{\textbf{0.887}} & \textbf{0.910} \\ \midrule[1pt]
\end{tabular}
  \label{table2}
\end{table*}

\subsection{Ablation Study}

\emph{\textbf{1) Time Efficiency Analysis}}

A common concern for second-order federated methods is the additional computational overhead induced by curvature estimation and matrix inversion. To assess the practical efficiency of FedRCO, we perform a comprehensive time analysis in Fig. \ref{4.2}, covering both fine-grained time breakdowns and wall-clock convergence behavior. As shown in Fig. \ref{4.2}(a), the matrix inversion in FedRCO accounts for only 6.4\% of the total training time, enabled by our efficient implementation and lazy inverse update strategy that amortizes the cost of expensive inversions across iterations. In contrast, FedPM spends 36.1\% of its time on probabilistic computations, indicating a substantially higher computational burden. Moreover, the increased local computation in FedRCO reduces the relative communication cost to 9.2\%, compared to 19.5\% for FedAvg, making FedRCO more computation-bound and well aligned with modern accelerator-equipped edge devices. Importantly, this modest per-round overhead translates into a decisive advantage in wall-clock performance. As shown in Fig. \ref{4.2}(b), FedRCO reaches 70\% accuracy within 1,000 seconds, while FedAvg and FedProx require over 10,000 seconds to approach comparable performance. FedRCO also exhibits a steep initial accuracy rise and maintains stable convergence throughout training, unlike FedRCO-ori, which converges quickly but degrades due to the absence of robust aggregation. Overall, these results demonstrate that FedRCO effectively breaks the conventional trade-off between convergence speed and computational cost, delivering both fast time-to-accuracy and consistently superior final performance in practice.

\emph{\textbf{2) Impact of the Inversion Frequency ($T_{inv}$).}}

Fig. \ref{4.4} evaluates the model's sensitivity to the inversion interval $T_{inv}$ across both communication rounds and wall-clock time. Contrary to the intuition that more frequent updates ($T_{inv}=20$ or $50$) would yield better results, the red curve ($T_{inv}=200$) consistently achieves the highest accuracy, which shows that very frequent updates result in lower steady-state accuracy. Meanwhile, excessively lazy updates lead to stale curvature information that fails to adapt to the changing loss landscape, causing sub-optimal convergence. In the time-accuracy plot, $T_{inv}=200$ reaches its peak accuracy faster than all other configurations. While $T_{inv}=500$ has theoretically lower amortized cost, its poor per-step progress results in a significantly longer time to reach target accuracy compared to $T_{inv}=200$. The results demonstrate that the Lazy Inverse Update is not just a computational necessity but also a performance stabilizer. A moderate interval effectively suppresses curvature noise and minimizes the computational cost, achieving the best trade-off between efficiency and accuracy.

\emph{\textbf{3) Stability Analysis of Gradient Anomaly Detection and Resilience}}

\begin{figure}[!t]
  \centering
  \includegraphics[width=3in]{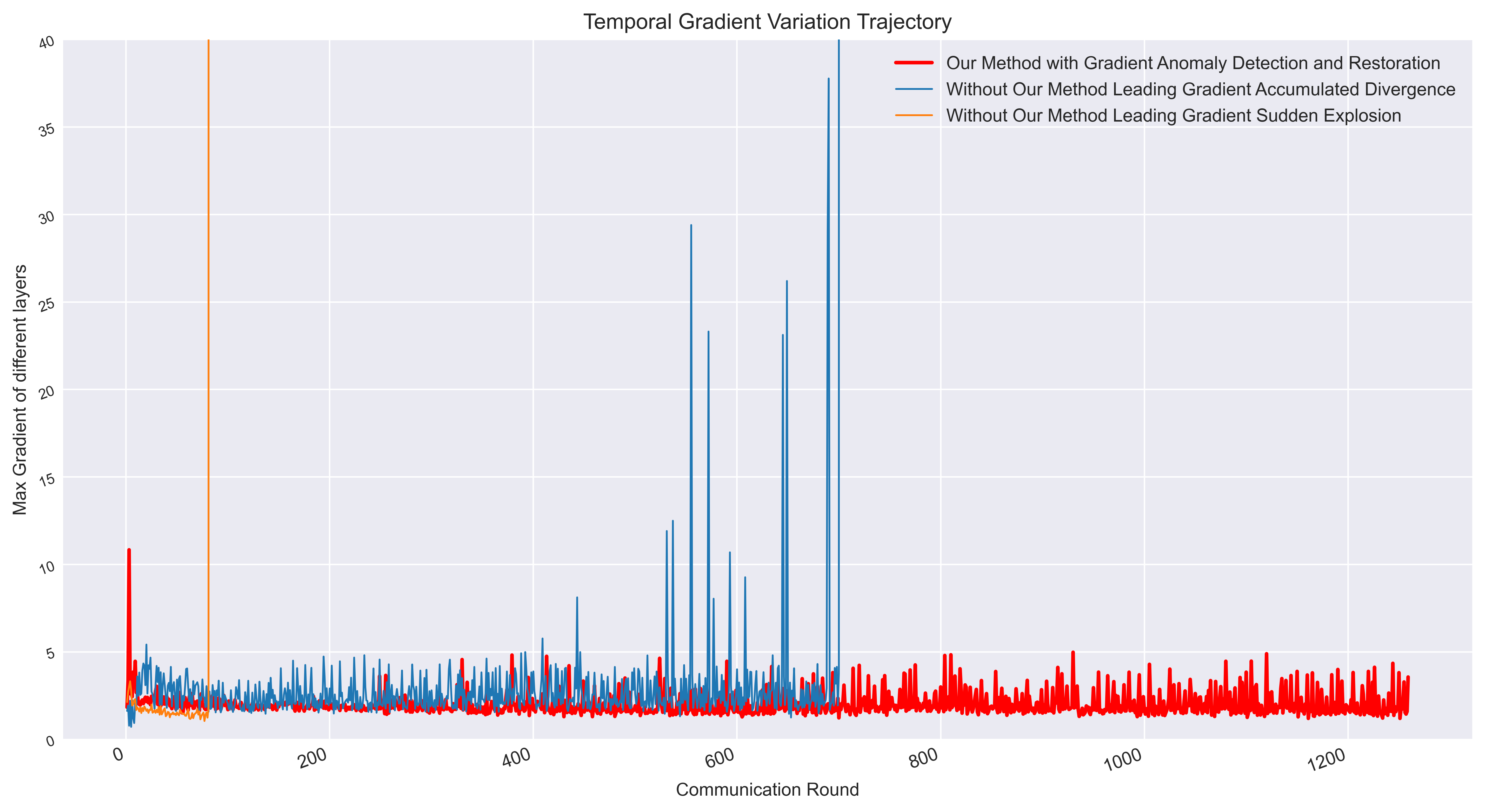}
  \caption{Visualization of gradient stability. The plot compares the gradient trajectories with (red) and without (blue/orange) the Gradient Monitor.}
  \label{4.5}
\end{figure}

To evaluate the effectiveness of the proposed gradient monitor, we visualize the temporal trajectory of the maximum gradient across layers during the training process in Fig. \ref{4.5}. In the absence of our monitoring mechanism, the second-order optimization framework exhibits high numerical instability: one scenario leads to a sudden gradient explosion (orange line) early in training, while another triggers accumulated divergence (blue line) in later rounds due to the noise amplification in the rank-deficient FIM. In contrast, our method, equipped with the gradient monitor (red line), successfully suppresses these extreme fluctuations and maintains the gradient magnitude within a stable, bounded range throughout the entire communication rounds. This ablation result empirically validates that the Gradient Monitor is indispensable for rectifying curvature-induced instability and ensuring the robust convergence of FedRCO.

\emph{\textbf{4) Impact of Client Number and Participation Ratio}}

We demonstrate the exceptional scalability of FedRCO through extensive experiments varying client populations and participation ratios. The results are shown in Table \ref{table2} as well as Fig. \ref{appd.1} to Fig. \ref{appd.21} in \textbf{Appendix E} for details. Unlike baselines that stagnate or fluctuate as network complexity increases, FedRCO consistently dominates across all configurations. Notably, when scaling to 100 clients, our method achieves a commanding $78.8\%$ accuracy on CIFAR-10—outperforming FedAvg and FedProx by a significant margin—while maintaining over $88\%$ on EMNIST. Furthermore, FedRCO proves impervious to participation sparsity, maintaining a decisive lead from low ($0.1$) to full ($1.0$) participation rates. These results conclusively validate that our robust curvature correction effectively neutralizes local drift, ensuring superior convergence and model quality in both massive, fragmented networks and stable, full-participation scenarios compared to all first-order and adaptive baselines.

\section{Conclusion}

We introduced \textbf{FedRCO}, a novel federated second-order optimization method that stabilizes curvature estimation in the presence of data heterogeneity. FedRCO addresses the instability through a robust correction mechanism and achieves high computational efficiency through a lazy update strategy. FedRCO dramatically accelerates convergence and accuracy across diverse non-IID settings, outperforming baselines with negligible overhead. The robustness and scalability of FedRCO make it a promising solution for efficient federated learning on decentralized, resource-constrained devices.

\bibliography{example_paper}
\bibliographystyle{icml2026}

\newpage
\appendix
\onecolumn
\section{Theoretical Analysis of Gradient Instability in Federated Second-Order Optimization}

In this section, we provide a theoretical analysis of the numerical instability observed when applying second-order optimization methods in a federated learning setting. We identify three primary sources of gradient explosion: (1) Rank deficiency due to small-batch approximation, (2) Curvature mismatch arising from non-IID data distributions, and (3) Error analysis caused by stale curvature information.

First, we provide all the notation used in the main text and in the appendices.

\begin{table}[H]
	\renewcommand{\arraystretch}{1.3}
\caption{Notations used in this paper}
	\centering
\begin{tabular}{c|c}
\toprule[1pt]
        \textbf{Notation}&\ \textbf{Description mainly used in the text}\\
        \midrule[1pt]
$\boldsymbol{\theta}$ & Model parameter, $\boldsymbol{\theta} \in \mathbb{R}^d$. \\
$\mathcal{L}(\boldsymbol{\theta}),\mathcal{L}_c(\boldsymbol{\theta})$ & Global loss function, local loss function of client $c$. \\
$C,c$ & Total number of clients, client index. \\
$T,t$ & Total number of communication rounds, communication round index. \\
$K,k$ & Total number of local epochs, local epoch index. \\
$n,n_c$ & Total data volume, local data volume, $n=\sum_{c=1}^Cn_c$. \\
$l$ & Layer index of model. \\
$\boldsymbol{H},$ & Hessian Matrix. \\
$\boldsymbol{F}, \hat {\boldsymbol{F}}$ & Fisher Information Matrix (FIM), the empirical FIM. \\
$\boldsymbol{A}$& Input activation matrix.\\
$\boldsymbol{G}$& Gradient matrix of the loss with respect to the layer’s pre-activation outputs.\\
$\boldsymbol{\Omega}$& Covariance of $\boldsymbol{A}$.\\
$\boldsymbol{\Gamma}$& Covariance of $\boldsymbol{G}$.\\
$\mathcal{D}_c$ & Local dataset held by client $c$. \\
$\boldsymbol{g}$ & Stochastic gradient vector, $\boldsymbol{g} = \nabla \mathcal{L}(\boldsymbol{\theta})$. \\
$T_{inv}$ & Frequency of inversion interval. \\
$\eta$ & Learning rate. \\
$ \pi,{\epsilon},\xi,\gamma$ & Parameters controlling Kronecker factors, Tikhonov damping, anomaly score, and aggregation. \\
        \midrule[1pt]
        \textbf{Notation}&\ \textbf{Description mainly used in the appendices}\\
        \midrule[1pt]
$B$ & Local mini-batch size. \\
$\lambda,\lambda_{min},\lambda_{max}$ & Eigenvalues, minimum eigenvalue, maximum eigenvalue. \\
$v,\delta$ & Direction eigenvector, local curvature. \\
$e,e_{dir},e_{drift}$ & Error, error in the preconditioned update direction, expected squared norm of the client drift. \\
$\kappa(\boldsymbol H)$ & Condition number of the Hessian. \\
$L$ & Lipschitz smoothness constant for assumption \ref{assu1}. \\
$\mu$ & Strong convexity constant for assumption \ref{assu2}. \\
$\sigma$ & Variance constant assumption \ref{assu3}. \\
$M$ & Gradient constant assumption \ref{assu4}. \\
$\mathcal{U}^t$ & Averaged aggregate update for communication round $t$. \\
$\theta^*,\mathcal{L}^*$ & Optimal model, global optimal value. \\
$\rho$ & Effective convergence decay rate, $\rho \approx \eta E \mu \lambda_{min}$. \\
\bottomrule[1pt]
\end{tabular}
\label{note}
\end{table}

\subsection{Preliminaries and Notation}

Let $\mathcal{L}(\boldsymbol\theta) = \mathbb{E}_{(x,y)\sim \mathcal{D}}[l(f(x; \boldsymbol\theta), y)]$ be the global loss function parameterized by $\boldsymbol\theta \in \mathbb{R}^d$. In Federated Learning, we have $C$ clients, where the $c$-th client minimizes a local loss $\mathcal{L}_c(\theta)$ over a local dataset $\mathcal{D}_c$.The standard second-order update rule for client $c$ in the communication round $t$ is given by:

\begin{align}
\boldsymbol\theta^{t+1} = \boldsymbol\theta^t - \eta (\hat{\boldsymbol{F}}_c + \epsilon \boldsymbol I)^{-1} \nabla \mathcal{L}_c(\boldsymbol\theta_t),
\end{align}

where $\hat{\boldsymbol{F}}_c$ is the empirical Fisher Information Matrix (FIM) or Hessian approximation, $\epsilon$ is the damping factor, and $\eta$ is the learning rate.

\subsection{Instability Induced by the Small-Batch Regime}

\textbf{Proposition \ref{prop1}.} \emph{\textbf{(Rank Deficiency)} When the mini-batch size $B$ is smaller than the model dimension $d$ ($B \ll d$), the empirical Fisher Matrix $\hat{\boldsymbol{F}}_c$ becomes rank-deficient. Without sufficient damping, sampling noise within the null space of $\hat{\boldsymbol{F}}_c$ can lead to unbounded parameter updates.}


\begin{proof}

The empirical Fisher matrix computed on a mini-batch of size $B$ is defined as:

\begin{align}
\hat{\boldsymbol{F}}_c = \frac{1}{B} \sum_{i=1}^B \nabla \log p(y_i|x_i; \boldsymbol\theta) \nabla \log p(y_i|x_i; \boldsymbol\theta)^\top.
\end{align}

Let $g_i = \nabla \log p(y_i|x_i;\boldsymbol \theta)$ denote the gradient. $\hat{\boldsymbol{F}}_c$ is a sum of the $B$ rank-1 matrices. Thus, $\text{rank}(\hat{F}_c) \le B$. Consider the spectral decomposition of $\hat{\boldsymbol{F}}_c$:

\begin{align}
    \hat{\boldsymbol{F}}_c = \sum_{j=1}^d \lambda_j u_j u_j^\top,
\end{align}

where the eigenvalues $\lambda_1 \ge \lambda_2 \ge \dots \ge \lambda_d \ge 0$, and $u_j$ denote the eigenvectors. Since $\text{rank}(\hat{\boldsymbol{F}}_c) \le B$, we have $\lambda_j = 0$ for all $j > B$.The preconditioned gradient update $\Delta\boldsymbol \theta = -(\hat{\boldsymbol{F}}_c + \epsilon I)^{-1} g_c$ can be expanded on the eigenbasis:

\begin{align}
    \Delta  \boldsymbol \theta = -\sum_{j=1}^d \frac{u_j^\top g_c}{\lambda_j + \epsilon} u_j.
\end{align}

Splitting this sum into the non-zero subspace and the null space:
\begin{align}
    \Delta \boldsymbol \theta &= -\underbrace{\sum_{j=1}^B \frac{u_j^\top g_c}{\lambda_j + \epsilon} u_j}_{\text{Signal Component}} - \underbrace{\sum_{j=B+1}^d \frac{u_j^\top g_c}{\epsilon} u_j}_{\text{Noise/Null Component}},\notag\\
\|\Delta \boldsymbol \theta\|_2^2 &\ge \sum_{j=B+1}^d \frac{(u_j^\top g_c)^2}{\epsilon^2}.
\end{align}

The local stochastic gradient $g_c$ typically contains noise. In high-dimensional spaces, this noise vector is almost guaranteed to have a non-zero projection onto the null space of $\hat{\boldsymbol{F}}_c$. For the null component, the scaling factor is $1/\epsilon$. As $\epsilon \to 0$, $\|\Delta\boldsymbol \theta\| \to \infty$. This confirms that in small-batch FL settings, the update vector is dominated by noise in the unconstrained directions, causing the gradient explosion detected by our Look-Ahead Gradient Monitor.
\end{proof}

\subsection{ Divergence due to non-IID Curvature Mismatch}

\textbf{Proposition \ref{prop2}.} \emph{\textbf{(Curvature Mismatch)} Let $\mathcal{L}_c$ be the local objective and $\mathcal{L}_{global}$ be the global objective.  If the local curvature underestimates the global curvature along the update direction, for example, the local landscape appears flatter while the global is steep, the quadratic penalty in the global objective can become unbounded, leading to divergence.}

\begin{proof}

Consider the quadratic approximation of the global loss function at $ \boldsymbol\theta^t$:

\begin{align}
    \mathcal{L}_{global}( \boldsymbol\theta^{t+1}) \approx \mathcal{L}_{global}( \boldsymbol\theta^t) + \nabla \mathcal{L}_{global}^\top(\boldsymbol\theta^t) (\Delta \boldsymbol \theta) + \frac{1}{2} (\Delta  \boldsymbol \theta)^\top  \nabla^2 \mathcal{L}_{global}(\boldsymbol\theta^t)  (\Delta  \boldsymbol \theta).
\end{align}

Substituting the local second-order step $\Delta \boldsymbol \theta = -\boldsymbol H_c^{-1} g_c$ (assuming $\epsilon=0$ for simplicity), the quadratic penalty term is:

\begin{align}
\frac{1}{2} (\Delta  \boldsymbol \theta)^\top  \nabla^2 \mathcal{L}_{global}(\boldsymbol\theta^t)  (\Delta  \boldsymbol \theta) &= \frac{1}{2} g_c^\top  \boldsymbol H_c^{-1}\sum_{c=1}^C \frac{n_c}{n} \nabla^2 \mathcal{L}_c(\boldsymbol\theta^t) H_c^{-1} g_k \notag \\
&=\frac{1}{2} g_c^\top\boldsymbol  H_c^{-1}\boldsymbol  H_{global}\boldsymbol  H_c^{-1} g_c,
\end{align}

where $n_c$ is the data volume of $c$ clients $n=\sum_{c=1}^Cn_c$, and $\boldsymbol H_c$ and $\boldsymbol H_{global}$ denote the Hessian of the client $c$ and global model, respectively. Let $v$ be a direction eigenvector where data is sparse on client $c$ but dense globally.
Considering the local flatness, the client lacks information in direction $v$, so the local curvature is small: $v^\top\boldsymbol  H_c v = \delta \approx 0$. Consequently, the inverse curvature is large: $v^\top \boldsymbol H_c^{-1} v \approx 1/\delta$. Considering the global sharpness, this direction is well-constrained: $v^\top \boldsymbol H_{global} v = Y \gg 0$. Analyzing the matrix product $\boldsymbol  H_c^{-1} \boldsymbol H_{global}\boldsymbol  H_c^{-1}$ along direction $v$:

\begin{align}
    v^\top \boldsymbol  H_c^{-1} \boldsymbol H_{global}\boldsymbol  H_c^{-1} v \approx \frac{1}{\delta} \cdot Y \cdot \frac{1}{\delta} = \frac{Y}{\delta^2}.
\end{align}

The global loss increase is bounded by:

\begin{align}
    \mathcal{L}_{global}(\boldsymbol\theta_{t+1}) - \mathcal{L}_{global}(\boldsymbol\theta_t) \gtrsim -\eta \|\nabla \mathcal{L}_{global}\| \cdot \frac{1}{\delta} + \frac{1}{2} \eta^2 \frac{Y}{\delta^2}.
\end{align}

Since $\delta$ is in the denominator squared for the penalty term but linear for the descent term, for non-IID which represents sufficiently small $\delta$, the quadratic penalty dominates:

\begin{align}
    \lim_{\delta \to 0} \left( \frac{Y}{\delta^2} \right) = \infty.
\end{align}

This proves that trusting a locally inverted Hessian on non-IID data can lead to catastrophic increases in global loss, necessitating the robust resilience mechanism proposed in our method.
\end{proof}

\subsection{Perturbation Analysis of Stale Preconditioners}

In our method, we update the inverse matrix every $T_{inv}$ steps to reduce computational cost. Here we analyze the error introduced by this delay. Let $\boldsymbol F_t$ be the true Fisher matrix at step $t$, and $\boldsymbol F_{old}$ be the stale matrix computed $\tau$ steps ago. We can model this as a perturbation:

\begin{align}
   \boldsymbol  F_{old} =\boldsymbol  F_t +\boldsymbol  E,
\end{align}

where $\boldsymbol E$ is the error matrix due to the change in $\boldsymbol \theta$ over $\tau$ steps. We are interested in the error of the inverse $(\boldsymbol F_t + \boldsymbol E)^{-1}$. Using a first-order Neumann series expansion $(\boldsymbol A+\boldsymbol E)^{-1} \approx \boldsymbol A^{-1} - \boldsymbol A^{-1}\boldsymbol E\boldsymbol A^{-1}$, the inverse of the stale matrix can be approximated as :

\begin{align}
    (\boldsymbol F_{old})^{-1} \approx\boldsymbol  F_t^{-1} -\boldsymbol  F_t^{-1}\boldsymbol  E\boldsymbol  F_t^{-1}.
\end{align}

As a result, the error in the preconditioned update direction can be expressed as

\begin{align}
   \boldsymbol  e_{dir} = \Delta \boldsymbol\theta_{stale} - \Delta\boldsymbol \theta_{true} \approx \eta \boldsymbol (F_t^{-1}\boldsymbol  E\boldsymbol  F_t^{-1}) \nabla \mathcal{L},
\end{align}

Taking the spectral norm yields the bound:

$$\|\boldsymbol e_{dir}\| \le \eta \|\boldsymbol F_t^{-1}\|^2 \|\boldsymbol E\| \|\nabla \mathcal{L}\|$$.

Use $\|\boldsymbol F_t^{-1}\| = \frac{1}{\lambda_{min}(\boldsymbol F_t)}$ for symmetric positive definite matrices. Thus:

\begin{align}
    \|\boldsymbol e_{dir}\| \propto \frac{1}{\lambda_{min}^2} \|\boldsymbol E\|.
\end{align}

This bound highlights that in ill-conditioned landscapes where $\lambda_{\min}(\boldsymbol F_t)$ approaches zero, stale curvature information can, in principle, lead to a large deviation in the update direction.

However, this analysis characterizes a worst-case bias induced by delayed curvature updates and does not account for the stochastic nature of mini-batch optimization. In FL settings with non-IID data and limited local computation, FIM estimates obtained from individual mini-batches exhibit high variance. Empirically, we observe that excessively frequent curvature inversions amplify this estimation noise: the preconditioner overfits to batch-specific fluctuations, resulting in unstable optimization trajectories.

In contrast, updating the inverse curvature at a moderate interval allows curvature statistics to aggregate information across multiple steps, effectively acting as a temporal smoothing mechanism. This variance reduction often dominates the bias introduced by delayed updates, leading to more stable and efficient optimization in practice. Our experiments consistently show that a moderately large update interval (e.g., \(T_{\text{inv}} = 200\)) outperforms more frequent updates (e.g., \(T_{\text{inv}} = 20\)), while simultaneously reducing computational overhead. These observations reveal a bias--variance trade-off in the choice of \(T_{\text{inv}}\), which our method exploits to achieve both stability and efficiency in federated second-order optimization.

\section{Theoretical Justification for the Efficacy of Second-Order in Federated Learning}

In this section, we provide the theoretical motivation for employing K-FAC in federated settings. We formally demonstrate that second-order preconditioning effectively mitigates the poor conditioning of the optimization landscape caused by data heterogeneity, leading to higher convergence rates compared to first-order methods.

\subsection{The Conditioning Problem in Federated Optimizations}

Consider the local objective function $\mathcal{L}_c(\boldsymbol\theta)$ for client $c$. In the quadratic approximation near a local minimum $\boldsymbol\theta^*$, the loss is governed by the Hessian matrix $\boldsymbol H = \nabla^2 \mathcal{L}_c(\boldsymbol\theta^*)$:

\begin{align}
     \mathcal{L}_c(\boldsymbol \theta) \approx \mathcal{L}_c(\boldsymbol \theta^*) + \frac{1}{2} (\boldsymbol \theta -\boldsymbol  \theta^*)^\top\boldsymbol  H (\boldsymbol \theta - \boldsymbol \theta^*).
\end{align}

The convergence speed of first-order gradient descent is fundamentally limited by the condition number $\kappa(\boldsymbol H)$ of the Hessian:

\begin{align}
    \kappa(\boldsymbol H) = \frac{\lambda_{max}(\boldsymbol H)}{\lambda_{min}(\boldsymbol H)},
\end{align}

where $\lambda_{max}$ and $\lambda_{min}$ are the largest and smallest eigenvalues of $H$.

\textbf{Proposition \ref{prop3}.}  \emph{
\textbf{(SGD Convergence Rate)} For a convex quadratic objective, the error contraction of SGD with optimal learning rate is bounded by:}

\begin{align}
    \|\boldsymbol \theta_{t+1} - \boldsymbol \theta^*\| \le \left( \frac{\kappa(\boldsymbol H) - 1}{\kappa(\boldsymbol H) + 1} \right) \|\theta_t - \theta^*\|.
\end{align}

In FL, data heterogeneity induces highly skewed loss landscapes. For wxample, if a client only has examples of class A but not B, the curvature along the B features is near zero, leading to $\lambda_{min} \to 0$, while the A features are steep. This leads to $\kappa(\boldsymbol H) \to \infty$. Consequently, the contraction factor $\frac{\kappa-1}{\kappa+1} \to 1$, meaning SGD convergence becomes arbitrarily slow.

\subsection{Geometric Correction via Fisher Preconditioning}

K-FAC approximates the NGD update, which uses the Fisher Information Matrix $\boldsymbol F$ as a preconditioner. The update rule is

\begin{align}
    \Delta \boldsymbol \theta = -\eta \boldsymbol F^{-1} \nabla \mathcal{L}.
\end{align}

\textbf{Proposition \ref{prop4}.} \emph{\textbf{(Affine Invariance and Isotropization)} Ideally, preconditioning by the Hessian or the Fisher transforms the geometry of the parameter space into an isotropic sphere.}

\begin{proof}
Let the optimization landscape around a local optimum $\theta^*$ be approximated by a quadratic function. Consider a linear change of basis $\phi = \boldsymbol{F}^{1/2}(\boldsymbol{\theta} - \boldsymbol{\theta}^*)$. In this transformed coordinate system, the Hessian of the loss $\mathcal{L}(\phi)$ becomes:

\begin{align}
    \mathcal{L}(\boldsymbol \phi) \approx \frac{1}{2}\boldsymbol  \phi^\top (\boldsymbol F^{-1/2} \boldsymbol H \boldsymbol F^{-1/2}) \boldsymbol \phi=\frac{1}{2}\boldsymbol  \phi^\top (\tilde{\boldsymbol{H}}) \boldsymbol \phi.
\end{align}

Under the standard assumption that the model distribution matches the data distribution near the optimum, we have $\boldsymbol{H} \approx \boldsymbol{F}$. Consequently, the transformed Hessian approximates the identity matrix:

\begin{align}
    \tilde{\boldsymbol H} = \boldsymbol F^{-1/2} \boldsymbol H \boldsymbol F^{-1/2} \approx \boldsymbol F^{-1/2} \boldsymbol F \boldsymbol F^{-1/2} = \boldsymbol I.
\end{align}

The condition number of the identity matrix $\boldsymbol I$ is:

\begin{equation}
\kappa(\tilde{\boldsymbol{H}}) \approx \kappa(\boldsymbol{I}) = 1
\end{equation}

Substituting $\kappa=1$ into the convergence bound from Proposition \ref{prop3}:

\begin{align}
    \frac{\kappa(\boldsymbol H) - 1}{\kappa(\boldsymbol H) + 1} = \frac{1 - 1}{1 + 1} = 0.
\end{align}

Ideally, second-order optimization achieves quadratic convergence. In practice, K-FAC uses a block-diagonal approximation $\hat{F}$, leading to the condition number $\kappa_{new} \ll \kappa_{SGD}$. This implies that K-FAC can maximize the utility of each local training round, significantly reducing the required communication rounds.
\end{proof}

\section{Convergence Analysis of Second-order Federated Learning}

In this section, we provide a detailed convergence analysis of the proposed FedRCO  algorithm. We analyze the convergence behavior from two perspectives: (1) the local descent property on the client side, showing how the second-order preconditioner accelerates local objective minimization, and (2) the global convergence of the aggregated model on the server side, bounding the error accumulation caused by multiple local epochs and non-IID data.

\subsection{Problem Setup and Assumptions}

We consider the following federated optimization problem:

\begin{align}
\min_{\boldsymbol\theta \in \mathbb{R}^d} \mathcal{L}(\boldsymbol\theta) := \sum_{c=1}^c \frac{n_c}{n} \mathcal{L}_c(\boldsymbol\theta_c), \quad  \mathcal{L}_c(\boldsymbol\theta_c) = \mathbb{E}_{(x,y)\sim \mathcal{D}_c} [l({y} | {x}; \boldsymbol{\theta})],
\end{align}

where $n_c$ is the data volume of $c$ clients $n=\sum_{c=1}^Cn_c$, $C$ is the number of clients, and $\mathcal{D}_c$ is the local data distribution.

Let $\hat{\boldsymbol F}_c^{-1}(\boldsymbol\theta)$ denote the approximate inverse Fisher matrix for client $c$ at parameters $\theta$. The local update rule at communication round $t$ on client $c$ is:
\begin{align}
    \boldsymbol\theta^{t+1}_c = \boldsymbol\theta^t_c - \eta^t \hat{\boldsymbol F}_c^{-1}(\boldsymbol\theta^t_c) \nabla \mathcal{L}_c(\boldsymbol\theta^t_c).
\end{align}

To facilitate the analysis, we make the following standard assumptions, which are widely used in the analysis of second-order and federated optimization methods.

\begin{assumption}
    \textbf{(L-Smoothness)} Each local objective function $\mathcal{L}_c$ is $L$-smooth. For all $\boldsymbol \theta_1, \boldsymbol \theta_2 \in \mathbb{R}^d$:
\begin{align}
        \mathcal{L}_c(\boldsymbol\theta_1) \le \mathcal{L}_c(\boldsymbol\theta_2) + \nabla \mathcal{L}_c(\boldsymbol\theta_2)^\top (\boldsymbol\theta_1 - \boldsymbol\theta_2) + \frac{L}{2} \|\boldsymbol\theta_1 - \boldsymbol\theta_2\|^2.
\end{align}
\label{assu1}
\end{assumption}

\begin{assumption}
    \textbf{($\mu$-Strong Convexity) }The global objective function $\mathcal{L}$ is $\mu$-strongly convex. For all $\boldsymbol\theta_1, \boldsymbol\theta_2 \in \mathbb{R}^d$:
\begin{align}
    \mathcal{L}(\boldsymbol\theta_1) \ge \mathcal{L}(\boldsymbol\theta_2) + \nabla \mathcal{L}(\boldsymbol\theta_2)^\top (\boldsymbol\theta_1 - \boldsymbol\theta_2) + \frac{\mu}{2} \|\boldsymbol\theta_1 - \boldsymbol\theta_2\|^2.
\end{align}
\label{assu2}
\end{assumption}

\begin{assumption}
\textbf{(Unbiased Gradient and Bounded Variance)} For any client $c$, the stochastic gradient $g_c(\boldsymbol\theta)$ is an unbiased estimator of the local full-batch gradient $\nabla \mathcal{L}_c(\boldsymbol\theta)$, and its variance is bounded by $\sigma^2$:
\begin{align}
\mathbb{E}[g_c(\boldsymbol\theta)] = \nabla \mathcal{L}_c(\boldsymbol\theta), \quad \mathbb{E}[|g_c(\boldsymbol\theta) - \nabla \mathcal{L}_c(\boldsymbol\theta)|^2] \le \sigma^2.
\end{align}
\label{assu3}
\end{assumption}

\begin{assumption}
    \textbf{(Bounded Gradients)} The expected squared norm of stochastic gradients is bounded:
\begin{align}
 \mathbb{E}[\|\nabla \mathcal{L}_c(\boldsymbol\theta)\| ] \le M^2.
\end{align}
\label{assu4}
\end{assumption}

\begin{assumption}
    \textbf{(Bounded Preconditioner Spectrum)} The approximate inverse Fisher matrix $\hat{F}_c^{-1}(\boldsymbol\theta)$ used in K-FAC satisfies the following eigenvalue bounds for all $c, \theta$:

\begin{align}
    \lambda_{min} \boldsymbol I \preceq \hat{F}_c^{-1}(\boldsymbol\theta) \preceq \lambda_{max} \boldsymbol I.
\end{align}

where $\lambda_{max}$ and $\lambda_{min}$ are the largest and smallest eigenvalues which have $0 < \lambda_{min} \le \lambda_{max}$.
\label{assu5}
\end{assumption}

\subsection{Client-Side Analysis: Local Descent Lemma}

First, we prove that the second-order update ensures a sufficient decrease in the local objective function value, effectively accelerating convergence compared to SGD.

\textbf{Theorem \ref{theo5.1}} \emph{    \textbf{(Local Descent with Preconditioning)} Under Assumptions \ref{assu1} and \ref{assu5}, for a learning rate $\eta$ satisfying $\eta \le \frac{\lambda_{min}}{L \lambda_{max}^2}$, the expected decrease in the local objective $\mathcal{L}_c$ for a single K-FAC step is lower bounded by:}
\begin{align}
    \mathbb{E}[\mathcal{L}_c(\boldsymbol\theta^{t+1}_c)] - \mathcal{L}_c(\boldsymbol\theta_c^t) \le - \eta \frac{\lambda_{min}}{2} \|\nabla \mathcal{L}_c(\boldsymbol\theta_c^t)\|^2 + \eta^2 \frac{L \lambda_{max}^2 \sigma^2}{2}.
\end{align}

\begin{proof}

From the $L$-smoothness of $\mathcal{L}_c$:
\begin{align}
    \mathcal{L}_c(\boldsymbol\theta^{t+1}) \le \mathcal{L}_c(\boldsymbol\theta^t) + \nabla \mathcal{L}_c(\boldsymbol\theta^t)^\top (\boldsymbol\theta^{t+1} -\boldsymbol \theta^t) + \frac{L}{2} \|\boldsymbol\theta^{t+1} - \boldsymbol\theta^t\|^2.
\end{align}

Substitute the second-order update rule $\boldsymbol\theta^{t+1} - \boldsymbol\theta^t = - \eta (\hat{\boldsymbol F}^{t})^{-1} g_t$, where $g_t$ is the stochastic gradient:
\begin{align}
    \mathcal{L}_c(\boldsymbol\theta^{t+1}) \le\mathcal{L}_c(\boldsymbol\theta^t) - \eta \nabla \mathcal{L}_c(\boldsymbol\theta^t)^\top (\hat{\boldsymbol F}^{t})^{-1} g_t + \frac{L \eta^2}{2} \|(\hat{\boldsymbol F}^{t})^{-1} g_t\|^2.
\end{align}

Taking the expectation over the stochastic noise and taking $\mathbb{E}[g_t] = \nabla \mathcal{L}_c(\theta^t)$:

$$\begin{aligned}
\mathbb{E}[\mathcal{L}_c(\boldsymbol\theta^{t+1})] &\le\mathcal{L}_c(\boldsymbol\theta^t) - \eta \nabla \mathcal{L}_c(\boldsymbol\theta^t)^\top (\hat{\boldsymbol F}^{t})^{-1} \nabla \mathcal{L}_c(\boldsymbol\theta^t) + \frac{L \eta^2}{2} \mathbb{E}[\|(\hat{\boldsymbol F}^{t})^{-1} g_t\|^2].
\end{aligned}$$

Using Assumption \ref{assu5}, $(\hat{\boldsymbol F}^{t})^{-1} \succeq \lambda_{min} \boldsymbol I$ and $\|(\hat{\boldsymbol F}^{t})^{-1}\| \le \lambda_{max}$, the quadratic term can be described as:
\begin{align}
    \nabla \mathcal{L}_c(\boldsymbol\theta^t)^\top (\hat{\boldsymbol F}^{t})^{-1} \nabla \mathcal{L}_c(\boldsymbol\theta^t) \ge \lambda_{min} ||\nabla \mathcal{L}_c(\boldsymbol\theta^t)||^2.
\end{align}

And the variance term can be described as:

\begin{align}
    \mathbb{E}[\|(\hat{\boldsymbol F}^{t})^{-1} g_t\|^2] \le \lambda_{max}^2 \mathbb{E}[\|g_t\|^2] = \lambda_{max}^2 (\|\nabla \mathcal{L}_c(\boldsymbol\theta^t)\|^2 + \sigma^2).
\end{align}

Substituting the variance term and the quadratic term:

\begin{align}
    \mathbb{E}[\mathcal{L}_c(\boldsymbol\theta^{t+1})] \le \mathcal{L}_c(\boldsymbol\theta^t) - \eta \lambda_{min} \|\nabla \mathcal{L}_c(\boldsymbol\theta^t)\|^2 + \frac{L \eta^2 \lambda_{max}^2}{2} (\|\nabla \mathcal{L}_c(\boldsymbol\theta^t)\|^2 + \sigma^2).
\end{align}

Regrouping the terms involving gradient norm:
\begin{align}
    \mathbb{E}[\mathcal{L}_c(\boldsymbol\theta^{t+1})] \le \mathcal{L}_c(\boldsymbol\theta^t) - \eta \left( \lambda_{min} - \frac{\eta L \lambda_{max}^2}{2} \right) \|\nabla \mathcal{L}_c(\boldsymbol\theta^t)
    \|^2 + \frac{L \eta^2 \lambda_{max}^2 \sigma^2}{2}.
\end{align}

Since we choosing $\eta \le \frac{\lambda_{min}}{L \lambda_{max}^2}$ and  $\lambda_{min} \le \lambda_{max}$,  we ensure the descent term is negative. Specifically, setting $\eta$ sufficiently small such that $\lambda_{min} - \frac{\eta L \lambda_{max}^2}{2} \ge \frac{\lambda_{min}}{2}$, we obtain:

\begin{align}
    \mathbb{E}[\mathcal{L}_c(\boldsymbol\theta^{t+1})] \le \mathcal{L}_c(\boldsymbol\theta^t) - \frac{\eta \lambda_{min}}{2}
\|\nabla \mathcal{L}_c(\boldsymbol\theta^t)\|^2 + \frac{L \eta^2 \lambda_{max}^2 \sigma^2}{2}.
\end{align}
\end{proof}

This lemma proves that on the client side, second-order optimization guaranties a descent proportional to $\lambda_{min}$. The standard SGD is a special case where $\lambda_{min}=\lambda_{max}=1$. Second-order provides a significant advantage when the geometry is ill-conditioned, standard SGD would struggle with a large Lipschitz constant $L$, but our method effectively rescales the space.

\subsection{Server-Side Global Convergence Theorem}
We now employ the local descent to prove the convergence of the global model ${\boldsymbol\theta}_t$ after $T'$ communication rounds.

\textbf{Theorem \ref{theo5.2}}\emph{    \textbf{(Global Convergence Rate of FedRCO)} Under Assumptions \ref{assu1}-\ref{assu5}, second-order federated optimization converges to a neighborhood of the optimal solution $\boldsymbol\theta^*$. Specifically, for the global model ${\boldsymbol \theta}^{t}$, the error bound satisfies:}
\begin{align}
        \mathbb{E} \|{\boldsymbol\theta}^{t+1} - \boldsymbol\theta^*\|^2 \le (1 - \rho) \mathbb{E} \|{\boldsymbol\theta}^t -\boldsymbol \theta^*\|^2 + E,
\end{align}
\emph{    where $\rho \approx \eta \mu K \lambda_{min}$, and $E$ represent noise and heterogeneity terms.}

\begin{proof}
The global model update rule at round $t$ after $k$ local steps aggregating updates from $C$ clients is:
\begin{align}
    {\boldsymbol\theta}^{t+1} = \boldsymbol{\theta}^t - \eta \frac{1}{C} \sum_{c=1}^C \sum_{k=0}^{K-1} (\hat{\boldsymbol F}_{c,k}^t)^{-1} g_{c,k}^t.
\end{align}

Let $\mathcal{ U}^t = \frac{1}{C} \sum_{c=1}^C \sum_{k=0}^{K-1} (\hat{\boldsymbol F}_{c,k}^t)^{-1} g_{c,k}^t$ be the averaged aggregate update. We analyze the distance to the optimum $\boldsymbol\theta^*$:
\begin{align}
    \| \boldsymbol{\theta}^{t+1} - \boldsymbol\theta^* \|^2 = \| \boldsymbol{\theta}^t - \eta \mathcal{U}^t - \boldsymbol\theta^* \|^2
= \| \boldsymbol{\theta}^t - \boldsymbol\theta^* \|^2 - 2\eta \langle {\boldsymbol\theta}^t - \boldsymbol\theta^*, \mathcal{ U}_t \rangle + \eta^2 \| \mathcal{U}^t \|^2.
\end{align}

First, we analyze the bounding of the expectation of the contraction term. We focus on the term $- 2\eta \mathbb{E} \langle \boldsymbol{\theta}^t - \boldsymbol\theta^*, \mathcal{ U}^t \rangle$. The aggregate update $\mathcal{ U}^t$ essentially approximates the descent direction. Ideally, we want $\mathcal{ U}^t \approx K \cdot\boldsymbol  F^{-1} \nabla \mathcal{L}({\boldsymbol\theta}^t)$. However, gradients are computed at local perturbed points $\boldsymbol\theta_{c,k}^t$.
\begin{align}
    \mathcal{ U}^t = \frac{1}{C} \sum_{c=1}^C \sum_{k=0}^{K-1} (\hat{\boldsymbol F}_{c,k}^t)^{-1}  (\nabla \mathcal{L}_c({\boldsymbol\theta}^t) + \underbrace{\nabla \mathcal{L}_c(\boldsymbol\theta_{c,k}^t) - \nabla \mathcal{L}_c({\boldsymbol\theta}_t)}_{\text{Drift Error}} + \underbrace{g_{c,k}^t - \nabla \mathcal{L}_c(\boldsymbol\theta_{c,k}^t)}_{\text{Noise}}).
\end{align}

To rigorously bound the contraction term $\langle \boldsymbol{\theta}^t - \boldsymbol{\theta}^*, (\hat{\boldsymbol{F}}_{c,k}^t)^{-1} \nabla \mathcal{L}_c(\boldsymbol{\theta}^t) \rangle$, we apply the Mean Value Theorem. Specifically, we can express the gradient difference as $\nabla \mathcal{L}_c(\boldsymbol{\theta}^t) - \nabla \mathcal{L}_c(\boldsymbol{\theta}^*) = \tilde{\mathbf{H}} (\boldsymbol{\theta}^t - \boldsymbol{\theta}^*)$, where $\tilde{\mathbf{H}} = \nabla^2 \mathcal{L}_c(\tilde{\boldsymbol{\theta}})$ is the Hessian evaluated at some interpolation point.

Assuming $\nabla \mathcal{L}_c(\boldsymbol{\theta}^*) \approx 0$, the term becomes a quadratic form:
\begin{align}
   \langle \boldsymbol{\theta}^t - \boldsymbol\theta^*, (\hat{\boldsymbol F}_{c,k}^t)^{-1} \nabla \mathcal{L}_c({\boldsymbol\theta}^t) \rangle= \langle \boldsymbol{\theta}^t - \boldsymbol{\theta}^*, (\hat{\boldsymbol{F}}_{c,k}^t)^{-1} \tilde{\mathbf{H}} (\boldsymbol{\theta}^t - \boldsymbol{\theta}^*) \rangle
\end{align}
Since both the preconditioner $(\hat{\boldsymbol{F}}_{c,k}^t)^{-1}$ and the Hessian $\tilde{\mathbf{H}}$ are positive definite matrices (Assumption \ref{assu2} and \ref{assu5}), the product matrix has positive eigenvalues. We can thus lower bound this term using the minimum eigenvalue of the product matrix:
\begin{align}
   \langle \boldsymbol{\theta}^t - \boldsymbol\theta^*, (\hat{\boldsymbol F}_{c,k}^t)^{-1} \nabla \mathcal{L}_c({\boldsymbol\theta}^t) \rangle \ge \lambda_{min}\left( (\hat{\boldsymbol{F}}_{c,k}^t)^{-1} \tilde{\mathbf{H}} \right) \|\boldsymbol{\theta}^t - \boldsymbol{\theta}^*\|^2    \ge {\lambda_{min} {\mu}}\|\boldsymbol{\theta}^t - \boldsymbol{\theta}^*\|^2
\end{align}

Summing over $c$ where $\sum \nabla \mathcal{L}_c = \nabla \mathcal{L}$ and $k$:

\begin{align}
    \mathbb{E} \langle {\theta}^t - \theta^*, \mathcal{U}^t \rangle \ge K \lambda_{min} \mu \| {\theta}^t - \theta^* \|^2 .
\end{align}

Thus, the contraction term becomes:

\begin{align}
    -2\eta \mathbb{E} \langle \boldsymbol{\theta}^t - \boldsymbol\theta^*, \mathcal{U}^t \rangle \le -2 \eta\mu K \lambda_{min}  \| \boldsymbol{\theta}^t - \boldsymbol\theta^* \|^2.
\end{align}

Then we analyze the bounding of the quadratic term. We bound $\eta^2 \mathbb{E} \| \mathcal{U}_t \|^2$. Using the Cauchy-Schwarz inequality $\|\sum_{i=1}^n x_i\|^2 \le n \sum_{i=1}^n \|x_i\|^2$, we have:
\begin{align}
    \mathbb{E} \| \mathcal{U}_t \|^2=\left\| \frac{1}{C} \sum_{c=1}^C  \sum_{k=0}^{K-1} (\hat{\boldsymbol F}_{c,k}^t)^{-1} g_{c,k}^t \right\|^2 \le  K \sum_{k=0}^{K-1} \left\| (\hat{\boldsymbol F}_{c,k}^t)^{-1} g_{c,k}^t \right\|^2.
\end{align}

By the definition of the induced matrix norm, $\|Ax\| \le \|A\|_2 \|x\|$, where $\|A\|_2$ corresponds to the largest singular value or eigenvalue for symmetric matrices. Per Assumption \ref{assu5}, $\|(\hat{F}_{c,k}^t)^{-1}\|_2 \le \lambda_{max}$, we have

\begin{align}
    \left\| (\hat{\boldsymbol F}_{c,k}^t)^{-1} g_{c,k}^t \right\|^2 \le \| (\hat{\boldsymbol F}_{c,k}^t)^{-1} \|_2^2 \| g_{c,k}^t \|^2 \le \lambda_{max}^2 \| g_{c,k}^t \|^2.
\end{align}

Substituting this back and taking the expectation:

\begin{align}
    \eta^2 \mathbb{E} \| \mathcal{U}_t \|^2 \le \eta^2 K \sum_{k=0}^{K-1} \lambda_{max}^2 \mathbb{E} \left[ \| g_{c,k}^t \|^2 \right].
\end{align}

Using Assumption \ref{assu4}, where $ \mathbb{E}[\|\nabla \mathcal{L}_c(\boldsymbol\theta)\| ] \le M^2.$:

\begin{align}
    & \eta^2 \mathbb{E} \| \mathcal{U}_t \|^2 \le \eta^2 K \sum_{k=0}^{K-1} \lambda_{max}^2 M^2 = \eta^2 K \cdot (K \lambda_{max}^2 M^2)=\eta^2
    K^2  \lambda_{max}^2 M^2.
\end{align}

Last, we combine the contraction term and quadratic term into the main equation:

\begin{align}
    \mathbb{E} \| \boldsymbol{\theta}^{t+1} - \boldsymbol\theta^* \|^2 \le \| \boldsymbol{\theta}^t - \boldsymbol\theta^* \|^2 - 2 \eta K \mu \lambda_{min} \| \boldsymbol{\theta}_t - \boldsymbol\theta^* \|^2  + \eta^2 K^2 \lambda_{max}^2 M^2.
\end{align}

Let $\rho = 2 \eta K \mu \lambda_{min}$.  Note that $\lambda_{min}$ accelerates convergence compared to first-order methods, where implicitly $\lambda_{min}=1$.

\begin{align}
\mathbb{E} \| \boldsymbol{\theta}_{t+1} - \boldsymbol\theta^* \|^2 \le (1 - \rho) \mathbb{E} \| \boldsymbol{\theta}_t - \boldsymbol\theta^* \|^2 + \mathcal{O}(\eta^2 K^2).
\end{align}

Applying this recursion over $T'$ communication rounds:
\begin{align}
    \mathbb{E} \| \boldsymbol{\theta}^{T'} - \boldsymbol\theta^* \|^2 \le (1 - \rho)^{T'} \Delta_0 + \frac{\mathcal{O}(\eta^2 K^2)}{\rho}.
\end{align}

The first term decays linearly to zero, while the second term represents the residual error floor due to stochastic noise and non-IID drift. The preconditioner $\lambda_{min}$ in $\rho$ effectively improves the condition number, leading to faster convergence than SGD.

\end{proof}

\section{Detailed Derivation of Optimization Bounds}

In this section, we rigorously derive the error bounds for both the client-side local updates and the server-side global aggregation. Our analysis explicitly incorporates the second-order preconditioner matrix, demonstrating how second-order information impacts the convergence trajectory compared to standard SGD.

\subsection{Preliminaries}
We retain the standard \textbf{Assumptions} from \textbf{Appendix C}, and we explicitly define the properties of the K-FAC preconditioner. Let $\boldsymbol F_{c,k}^{-1}$ be the K-FAC preconditioner for client $c$ at step $k$. The local update rule is:

\begin{align}
    \boldsymbol \theta_{c, k+1}^t =\boldsymbol  \theta_{c, k}^t - \eta \boldsymbol F_{c,k}^{-1} \nabla \mathcal{L}_c(\boldsymbol \theta_{c, k}^t),
\end{align}

where $t$ indexes the communication round and $k$ indexes the local epoch ($k \in \{0, \dots, K-1\}$).

\subsection{Client-Side Bound: Analyzing Local Drift}
A key challenge in FL is the client drift caused by performing several local steps before aggregation. We define the drift at step $t$ within communication round $E$ as $\Delta^t_c = \boldsymbol\theta^t_c - \boldsymbol\theta^t$, where $\boldsymbol\theta^t$ is the virtual global model. The Client Drift measures how far the local model deviates from the global model after $K$ steps of local training. This drift is the primary source of noise in Federated Learning.

\textbf{Theorem \ref{theo.d1}}\emph{    \textbf{(Client Drift Bound)} Under Assumptions \ref{assu1}, \ref{assu3}, \ref{assu4} and \ref{assu5}, the expected squared norm of the client drift after $K$ local steps is bounded by:}
\begin{align}
    {e}_{drift} = \mathbb{E} \left[ \left\| \boldsymbol\theta_{c, K}^t - \boldsymbol\theta^t \right\|^2 \right] \le 2 K^2 \eta^2 \lambda_{max}^2 \sigma^2 + 2 K^2 \eta^2 \lambda_{max}^2 M^2.
\end{align}

\begin{proof}

The accumulated parameter change after $K$ steps is the sum of local updates:
\begin{align}
    \boldsymbol\theta_{c, k}^t - \boldsymbol \theta^t = - \sum_{k=0}^{K-1} \eta (\boldsymbol F_{c,k}^t)^{-1} \nabla \mathcal{L}_c(\boldsymbol    \theta_{c, k}^t).
\end{align}

Taking the squared norm and expectation:

\begin{align}
    \mathbb{E} \left[ \left\| \boldsymbol\theta_{c, k}^t - \boldsymbol\theta^t \right\|^2 \right] = \eta^2 \mathbb{E} \left[ \left\| \sum_{k=0}^{K-1} (\boldsymbol F_{c,k}^t)^{-1} \nabla \mathcal{L}_c(\boldsymbol\theta_{c, k}^t) \right\|^2 \right].
\end{align}

Using the Jensen's inequality $\|\sum_{i=1}^n x_i\|^2 \le n \sum_{i=1}^n \|x_i\|^2$:

\begin{align}
       \mathbb{E} \left[ \left\| \boldsymbol\theta_{c, k}^t - \boldsymbol\theta^t \right\|^2 \right]\le \eta^2 K \sum_{k=0}^{K-1} \mathbb{E} \left[ \left\| (\boldsymbol F_{c,k}^t)^{-1} \nabla  \mathcal{L}_c(\boldsymbol\theta_{c, k}^t) \right\|^2 \right].
\end{align}

Now we apply the bound on the preconditioner $\boldsymbol F^{-1}$. Since $\|\boldsymbol F^{-1} v\| \le \|\boldsymbol F^{-1}\|_2 \|v\| \le \lambda_{max} \|v\|$:
\begin{align}
    \mathbb{E} \left[ \left\| \boldsymbol\theta_{c, k}^t - \boldsymbol\theta^t \right\|^2 \right]
    \le \eta^2 K \lambda_{max}^2 \sum_{k=0}^{-1} \mathbb{E} \left[ \left\| \nabla  \mathcal{L}_c(\boldsymbol\theta_{c, k}^t) \right\|^2 \right].
\end{align}

We decompose the stochastic gradient into the true gradient and variance:

\begin{align}
    \mathbb{E} [\|g\|^2] = \mathbb{E} [\|g - \nabla \mathcal{L} + \nabla \mathcal{L}\|^2] \le 2 \mathbb{E}[\|g - \nabla \mathcal{L}\|^2] + 2 \|\nabla \mathcal{L}\|^2 \le 2\sigma^2 + 2M^2.
\end{align}

Substituting this back:

\begin{align}
    e_{drift} \le \eta^2 K \sum_{k=0}^{K-1} \lambda_{max}^2 (2\sigma^2 + 2M^2) = 2 K^2 \eta^2 \lambda_{max}^2 (\sigma^2 + M^2).
\end{align}

\end{proof}

This result explicitly shows that the drift is proportional to $\lambda_{max}^2$. Suppose the Hessian approximation becomes singular, where very small eigenvalues $\to$ huge inverse eigenvalues $\lambda_{max}$, the drift explodes quadratically, which mathematically justifies the gradient normalization and robust resilience in FedRCO.


\subsection{Server-Side Bound: One-Round Convergence Guarantee}
Now we analyze how the global loss decreases after one round of aggregation.

\begin{theorem}
    \textbf{(One-Round Descent)} Let the global aggregation be $\boldsymbol\theta^{t+1} = \boldsymbol\theta^t + \Delta \boldsymbol\theta$, where $\Delta \boldsymbol\theta = \frac{1}{C} \sum_c (\boldsymbol \theta_{c,K}^t - \boldsymbol \theta^t)$. Under L-smoothness, the global objective improves as:

\begin{align}
    \mathbb{E}[\mathcal{L}(\boldsymbol\theta^{t+1})] \le \mathcal{L}(\boldsymbol\theta^{t}) - \underbrace{2\eta K \mu \lambda_{min} }_{\text{Effective Decay}} (\mathcal{L}(\boldsymbol\theta^{t}) - \mathcal{L}^*) + \underbrace{\frac{L}{2} e_{drift}}_{\text{Drift Error}},
\end{align}

where  $\mathcal{L}^*$ is the global optimal value.
\label{theo.d3}
\end{theorem}

\begin{proof}
    By $L$-smoothness of the global objective $\mathcal{L}$:

\begin{align}
    \mathcal{L}(\boldsymbol\theta^{t+1}) \le \mathcal{L}(\boldsymbol\theta^{t}) + \nabla \mathcal{L}(\boldsymbol\theta^{t})^\top (\boldsymbol\theta_{t+1} - \boldsymbol\theta_t) + \frac{L}{2} \|\boldsymbol\theta_{t+1} - \boldsymbol\theta_t\|^2.
\end{align}

Let the aggregated update be $\bar{\Delta} = \frac{1}{C} \sum_c \sum_{k=0}^{K-1} -\eta (\boldsymbol F_{c,k}^t)^{-1} g_{c,k}^t$.

\begin{align}
   \mathcal{L}(\boldsymbol\theta^{t+1}) \le\mathcal{L}(\boldsymbol\theta^{t}) \underbrace{- \eta \nabla\mathcal{L}(\boldsymbol\theta^{t})^\top \mathbb{E}[\bar{\Delta}]}_{T_1} + \underbrace{\frac{L}{2} \mathbb{E}[\|\bar{\Delta}\|^2]}_{T_2}.
\end{align}

First, we analyze the bounding of the Descent Term $T_1$. Ideally, we want the update to align with the gradient.

\begin{align}
    T_1 \approx - \eta K \nabla \mathcal{L}(\boldsymbol\theta^{t})^\top \left( \frac{1}{C} \sum_c \boldsymbol F_c^{-1} \nabla \mathcal{L}_c(\boldsymbol\theta^{t}) \right).
\end{align}

Using the preconditioner property $v^\top G v \ge \lambda_{min} \|v\|^2$ and strong convexity:

\begin{align}
    T_1 \le - \eta K \lambda_{min} \|\nabla \mathcal{L}(\boldsymbol\theta^{t})\|^2.
\end{align}

Using the Polyak-Lojasiewicz condition: $\|\nabla \mathcal{L}(\boldsymbol\theta)\|^2 \ge 2\mu (\mathcal{L}(\boldsymbol\theta) - \mathcal{L}^*)$:

\begin{align}
    T_1 \le - 2 \eta K \mu \lambda_{min} ( \mathcal{L}(\boldsymbol\theta^{t}) - \mathcal{L}^*).
\end{align}

As $\lambda_{min}$ for K-FAC is much larger than standard SGD, it allows for a steeper descent.

Second, we analyze the bounding of the term $T_2$. $T_2$ is essentially the average drift error. Using convexity of norms:

\begin{align}
    \|\bar{\Delta}\|^2 = \left\| \frac{1}{C} \sum_{c} (\boldsymbol\theta_{c,K}^t - \boldsymbol\theta^t) \right\|^2 &\le \frac{1}{C} \sum_c \|\boldsymbol\theta_{c,K}^t - \boldsymbol\theta^t\|^2 = e_{drift}\notag\\
    \frac{L}{2} \mathbb{E}[\|\bar{\Delta}\|^2] &\le \frac{L}{2} e_{drift}.
\end{align}

Combining $T_1$ and $T_2$

\begin{align}
    \mathbb{E}[\mathcal{L}(\boldsymbol\theta^{t+1})  - \mathcal{L}^*] \le (1 - 2\eta K \mu \lambda_{min}) (\mathcal{L}(\boldsymbol\theta^{t})  - \mathcal{L}^*) + \frac{L}{2} e_{drift}.
\end{align}

Substituting the drift bound from Theorem \ref{theo.d1}:

\begin{align}
    \mathbb{E}[\mathcal{L}(\boldsymbol\theta^{t+1})  - \mathcal{L}^*] \le (1 - \rho)  (\mathcal{L}(\boldsymbol\theta^{t})  - \mathcal{L}^*) + L K^2 \eta^2 \lambda_{max}^2 (\sigma^2 + M^2),
\end{align}

where $\rho = 2\eta K \mu \lambda_{min}$.
\end{proof}

\subsection{Final Global Convergence Bound}

\textbf{Theorem \ref{theo5.4}} \emph{Recursively applying Theorem \ref{theo.d3} leads to the final convergence rate. After $T'$ communication rounds, the convergence of our method satisfies:}

\begin{align}
   \mathbb{E}[\mathcal{L}(\boldsymbol\theta^{T'})  - \mathcal{L}^*] \le \underbrace{(1 - \rho)^{T'} (\mathcal{L}(\boldsymbol\theta^{0})  - \mathcal{L}^*)}_{\text{Linear Decay}} + \underbrace{\frac{L K \eta \lambda_{max}^2 (\sigma^2 + M^2)}{2 \mu \lambda_{min}}}_{\text{Asymptotic Error Floor}}.
\end{align}

For the linear decay speed $\rho$, the rate is governed by $\rho \propto \lambda_{min}$. In ill-conditioned problems, FedRCO ensures $\lambda_{min}$ is bounded away from zero, while in SGD it can be arbitrarily small. This proves faster convergence.

For the error floor, the final error depends on the ratio $\frac{\lambda_{max}^2}{\lambda_{min}}$. If the preconditioner is unstable ($\lambda_{max} \to \infty$), the error floor explodes. Our method proposes the Gradient Monitor and Robust Resilience, which effectively clips $\lambda_{max}$, keeping the ratio $\frac{\lambda_{max}^2}{\lambda_{min}}$ small and controlled. This theoretically proves why our method achieves lower final loss.

\section{Experimental Details}
Here, we introduce all other details used in this paper. We used four NVIDIA 3090 GPUs with 96GB of memory, Intel(R) Xeon(R) Gold 6226R CPU at 2.90GHz with 16 cores, and 128GB of RAM. The software environment includes Python 3.9.2, Pytorch 2.7.1+cu126, and CUDA 12.2.

\subsection{Model Structure}

To simulate realistic resource-constrained federated learning environments, we employ a lightweight Convolutional Neural Network (CNN) architecture for both datasets. This streamlined design ensures that the model can be deployed on edge devices with limited computational power and memory while maintaining sufficient representative capacity for the classification tasks.

\textbf{CIFAR-10 Model:} For the 3-channel color images, the architecture begins with a convolutional layer featuring 16 filters and a $3\times3$ kernel, followed by a $2\times2$ max-pooling layer to reduce spatial dimensions. A second convolutional layer with 32 filters and a $3\times3$ kernel is then applied. The resulting feature maps are flattened and passed through two successive fully connected layers with 32 and 256 neurons, respectively. Finally, a softmax output layer is used to produce the probability distribution over the 10 categories.

\textbf{EMNIST Model:}: To maintain consistency in computational complexity across different data modalities, we adopt a mirrored structure for the EMNIST dataset. The input grayscale images are processed through the same configuration of two convolutional layers (16 and 32 filters) and max-pooling, followed by the two-tier linear layers (32 and 256 neurons). The output layer is adjusted to 62 units to accommodate the character and digit classes in the EMNIST dataset.

\textbf{Implementation Details:}
 Throughout the network, we use the Rectified Linear Unit (ReLU) as the activation function for all hidden layers to mitigate the vanishing gradient problem and accelerate convergence. No batch normalization is used to avoid the synchronization overhead and potential instability caused by non-IID data in federated settings. This minimalist design allows us to focus on evaluating the effectiveness of our proposed FedRCO in optimizing curvature information under strict resource constraints.

\begin{figure}[!t]
  \centering
  \includegraphics[width=7in]{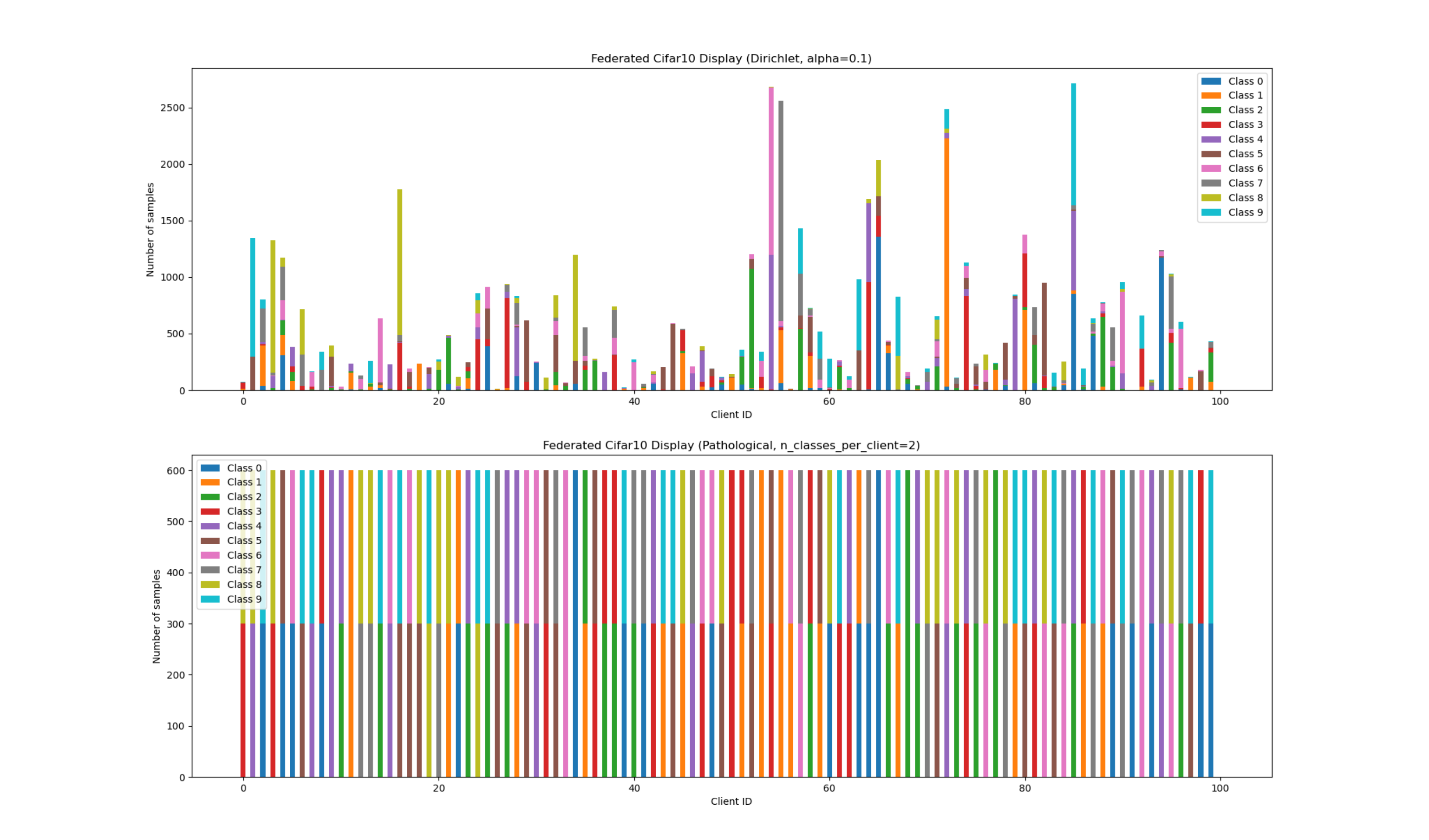}
  \caption{The distribution example. The first figure is the Dirichlet distribution on 100 clients with $Dir(\alpha)=0.1$, and the second figure is the Pathological distribution on 100 clients with 2 classes per client.}
  \label{distri}
\end{figure}

\subsection{Data Distribution}

To comprehensively evaluate the robustness of FedRCO against statistical heterogeneity, we employ two widely used non-IID data partitioning strategies: the Dirichlet-based distribution and the Pathological distribution. The distribution is shown in Fig. \ref{distri}.

\textbf{ Dirichlet-based Non-IID Distribution} The Dirichlet-based partitioning simulates a realistic scenario where the label distributions across clients are unbalanced. Specifically, for each class $k$, we sample a distribution vector $q_k \sim Dir(\alpha)$ and allocate a proportion of samples from class $k$ to client $i$ according to $q_{k,i}$. The concentration parameter $\alpha$ controls the degree of non-IID. We utilize $\alpha \in \{0.1, 0.5, 1.0\}$. A smaller $\alpha$ indicates a more extreme skewness, where most samples of a specific class are concentrated on only a few clients.

\textbf{ Pathological Non-IID Distribution}
The Pathological partitioning mimics a scenario where each client only has access to a limited subset of the total classes. This creates the situation where certain features or labels are entirely missing from most local datasets. Each client is randomly assigned a fixed number of unique labels. The samples belonging to these labels are then evenly distributed among the assigned clients. For CIFAR-10, we test with $2$ and $5$ classes per client. For EMNIST, due to the large number of 62 categories, we set each client to have $10$ or $30$ categories. The Pathological-2 setting on CIFAR-10 is particularly challenging, as local models tend to overfit to their own class sets, which may lead to drifting.

\subsection{Other ablation results}
Here we will show all the experimental results. For the Dirichlet setting, we sample client data from $Dir(\alpha)$ with \( \alpha \in \{0.1, 0.5, 1\} \). In the pathological setting, each client is restricted to a small subset of labels; specifically, clients are assigned \(\{2, 5\}\) labels in CIFAR-10 and \(\{10, 30\}\) labels in EMNIST. The number of clients varies in \(\{10, 50, 100\}\), and the percentage of client participation per communication round is set to \(\{0.1, 0.5, 0.8, 1\}\) for comparison. The communication round is set as 1600, and the local epoch is set as 20. The learning rate is set as 0.00625, batch size 32, EMA parameter $\alpha=0.95$, $\tilde{\boldsymbol{\nabla}}_{stable}$ is set as 10, and damping $\epsilon=0.03$ for all second-order methods. The parameters of comparison methods are all set to their optimal values.

The results displayed below from Fig. \ref{appd.1} to Fig. \ref{appd.21} are calculated on the average of all participants' clients. The upper left is the test accuracy, the upper right is the train accuracy, the lower left is the train loss, and the lower right is the test accuracy measured in real-time.

\begin{figure}[H]
  \centering
  \includegraphics[width=6in]{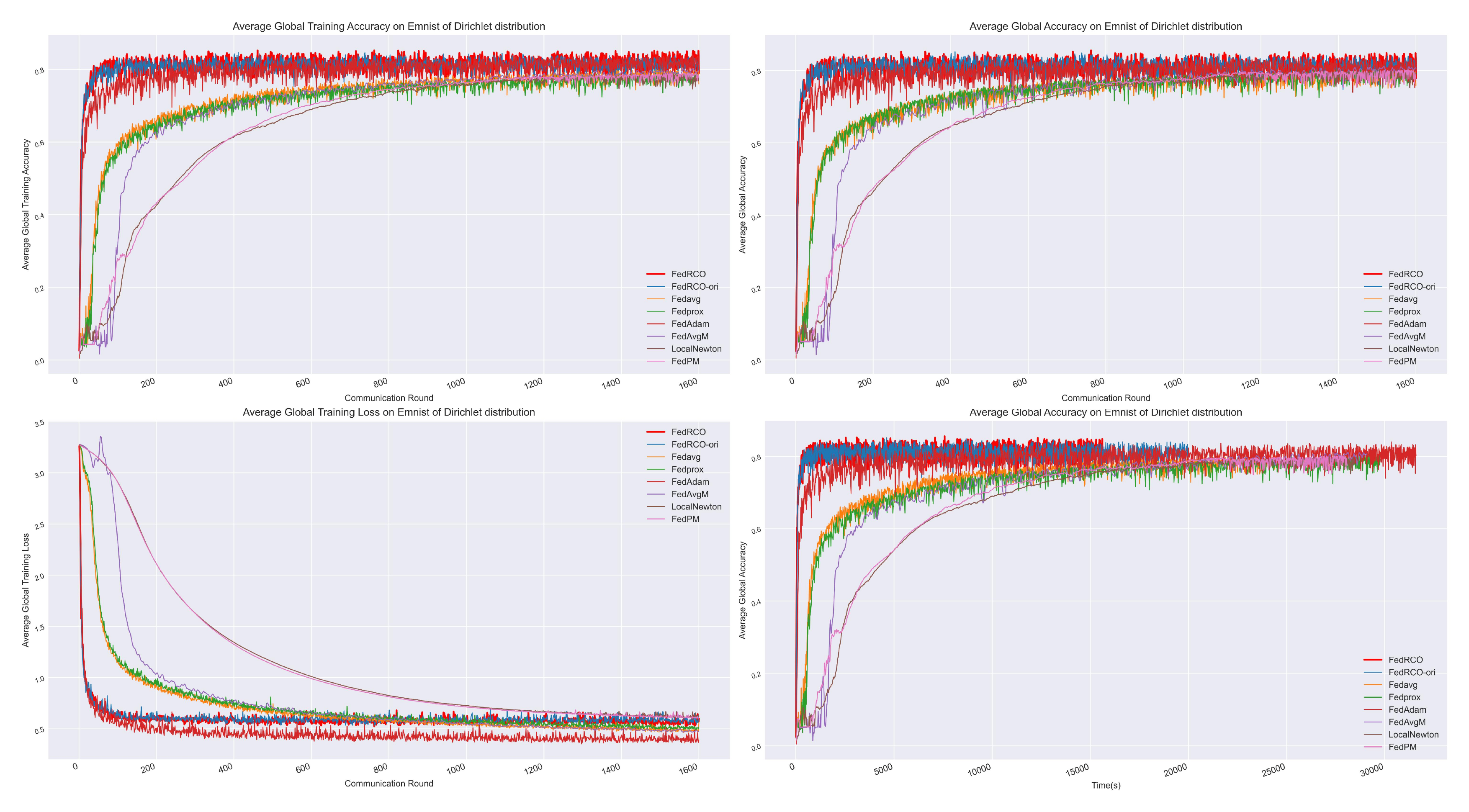}
  \caption{Dataset: EMNIST, Data distribution: Dirichlet 0.1, Party ratio: 0.1, Number clients: 100.}
  \label{appd.1}
\end{figure}

\begin{figure}[]
  \centering
  \includegraphics[width=6in]{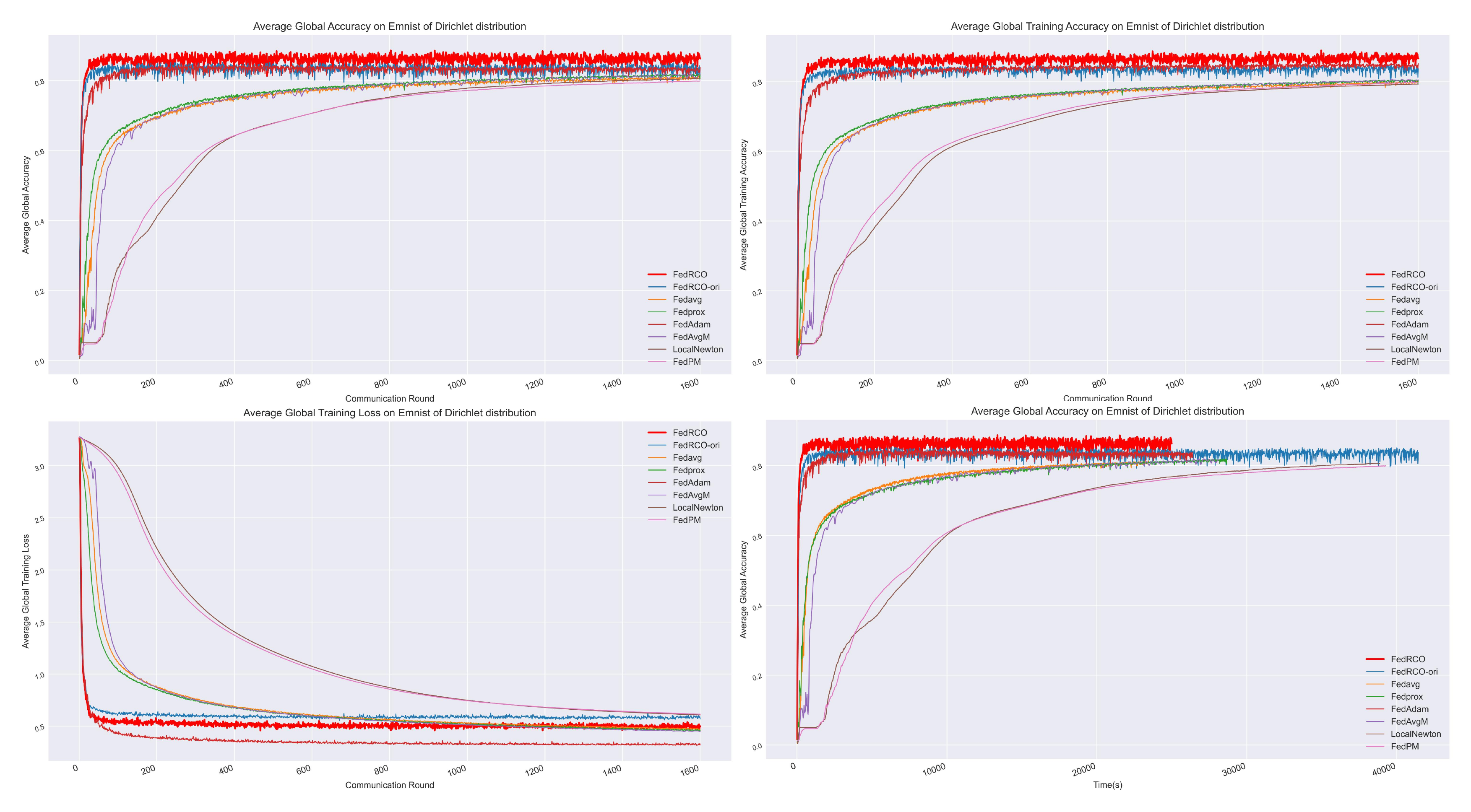}
  \caption{Dataset: EMNIST, Data distribution: Dirichlet 0.1, Party ratio: 0.5, Number clients: 100.}
  \label{appd.2}
\end{figure}

\begin{figure}[]
  \centering
  \includegraphics[width=6in]{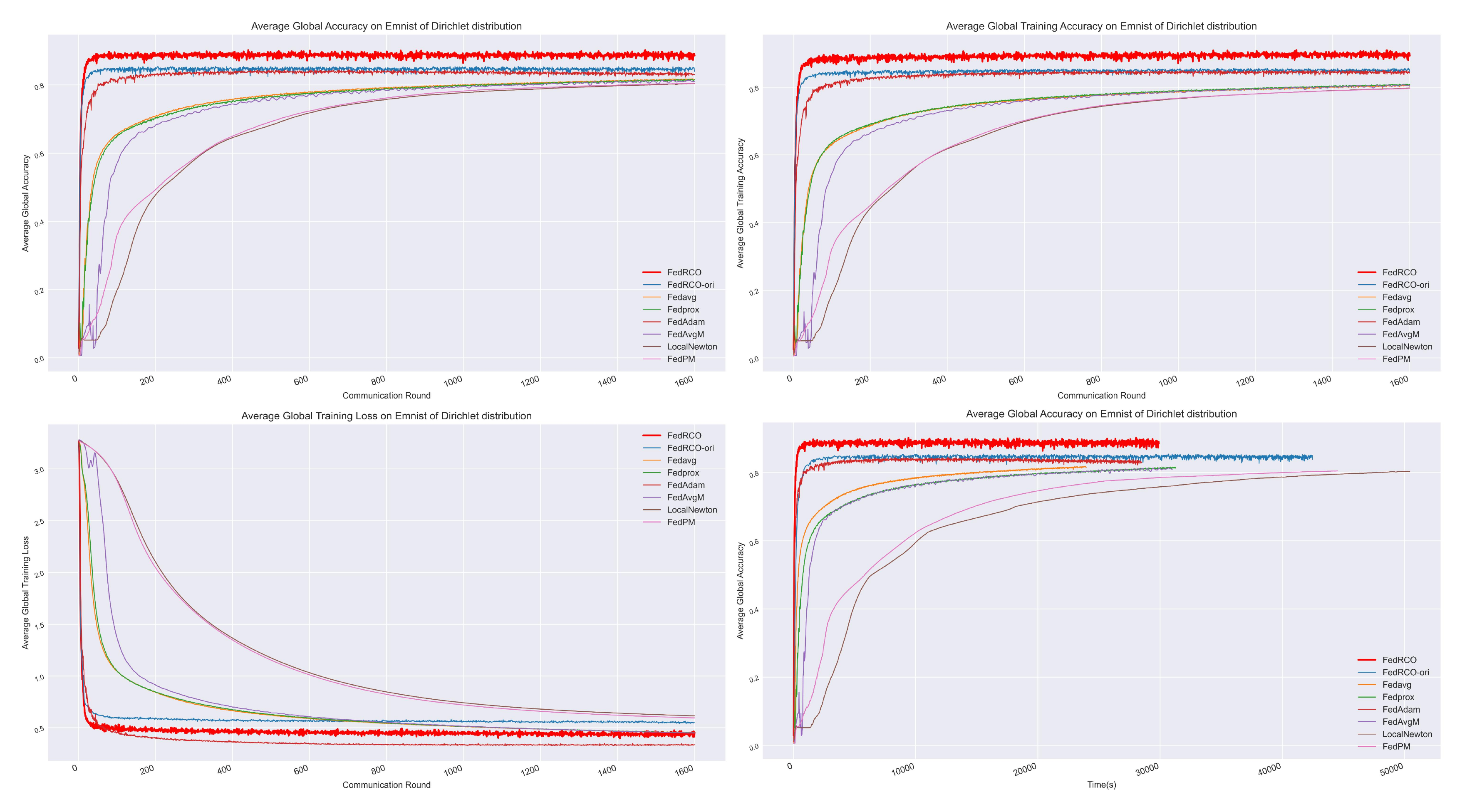}
  \caption{Dataset: EMNIST, Data distribution: Dirichlet 0.1, Party ratio: 0.8, Number clients: 100.}
  \label{appd.3}
\end{figure}

\begin{figure}[]
  \centering
  \includegraphics[width=6in]{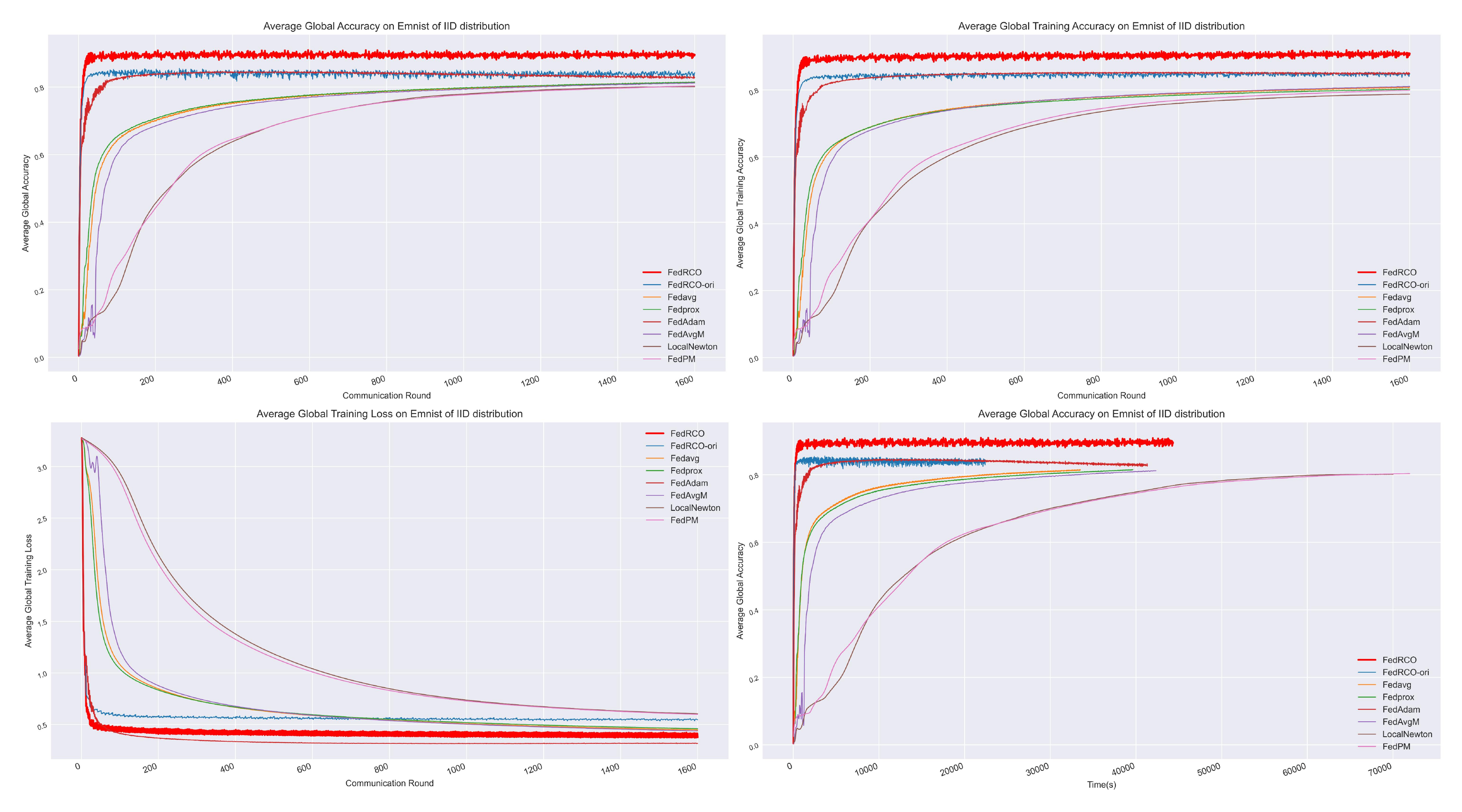}
  \caption{Dataset: EMNIST, Data distribution: Dirichlet 0.1, Party ratio: 1, Number clients: 100.}
  \label{appd.4}
\end{figure}

\begin{figure}[]
  \centering
  \includegraphics[width=6in]{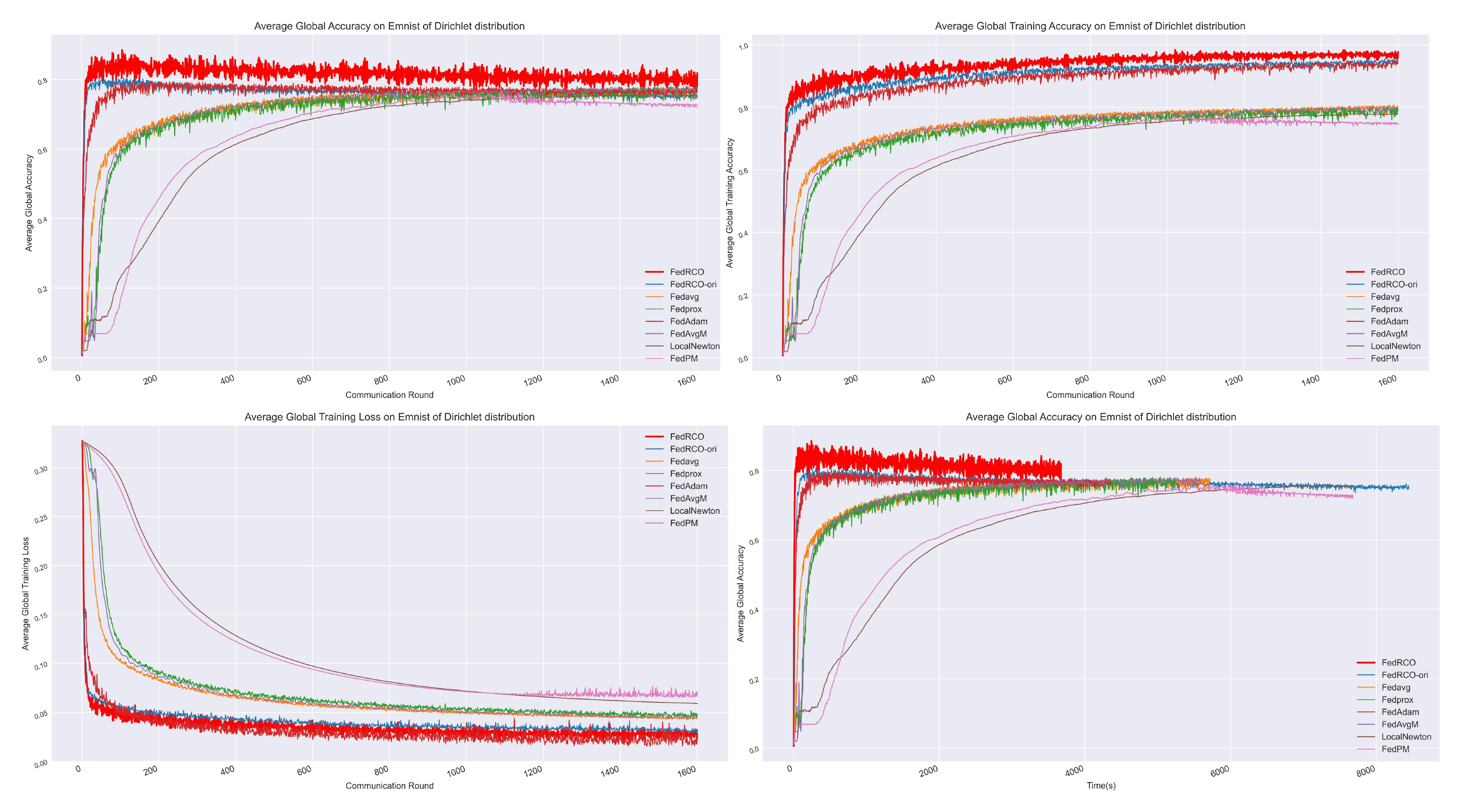}
  \caption{Dataset: EMNIST, Data distribution: Dirichlet 0.1, Party ratio: 0.8, Number clients: 10.}
  \label{appd.5}
\end{figure}

\begin{figure}[]
  \centering
  \includegraphics[width=6in]{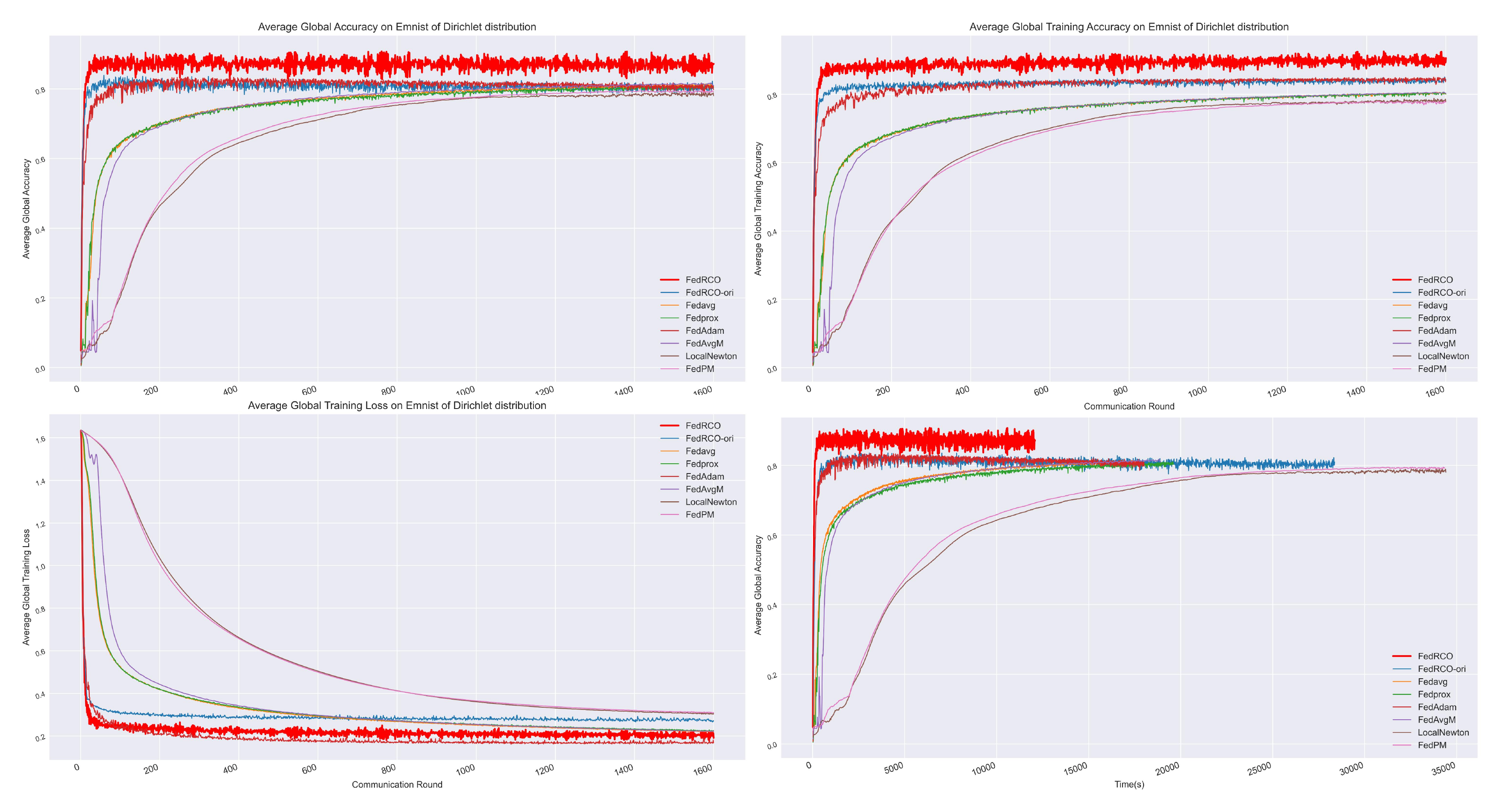}
  \caption{Dataset: EMNIST, Data distribution: Dirichlet 0.1, Party ratio: 0.8, Number clients: 50.}
  \label{appd.6}
\end{figure}

\begin{figure}[]
  \centering
  \includegraphics[width=6in]{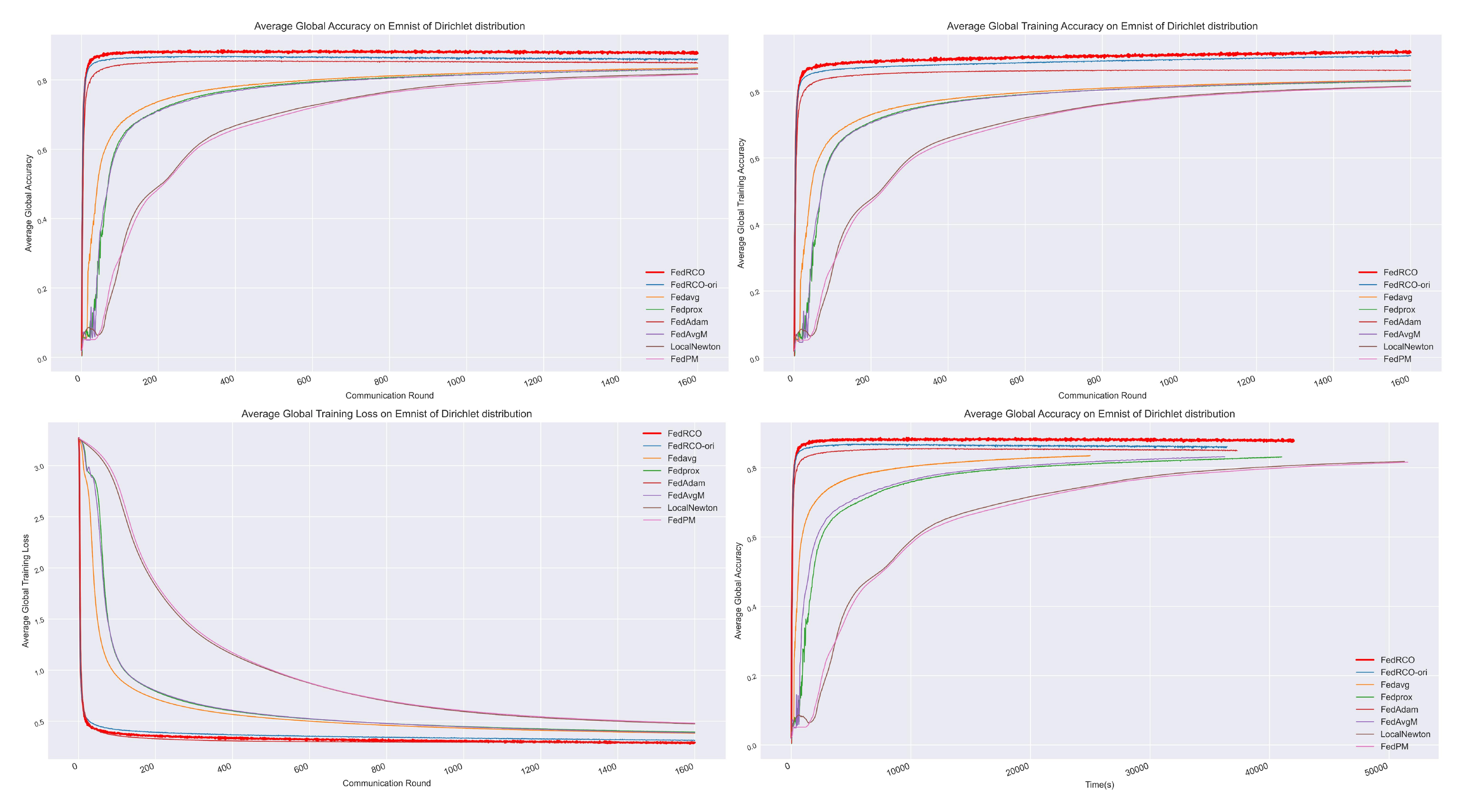}
  \caption{Dataset: EMNIST, Data distribution: Dirichlet 0.5, Party ratio: 0.8, Number clients: 100.}
  \label{appd.7}
\end{figure}

\begin{figure}[]
  \centering
  \includegraphics[width=6in]{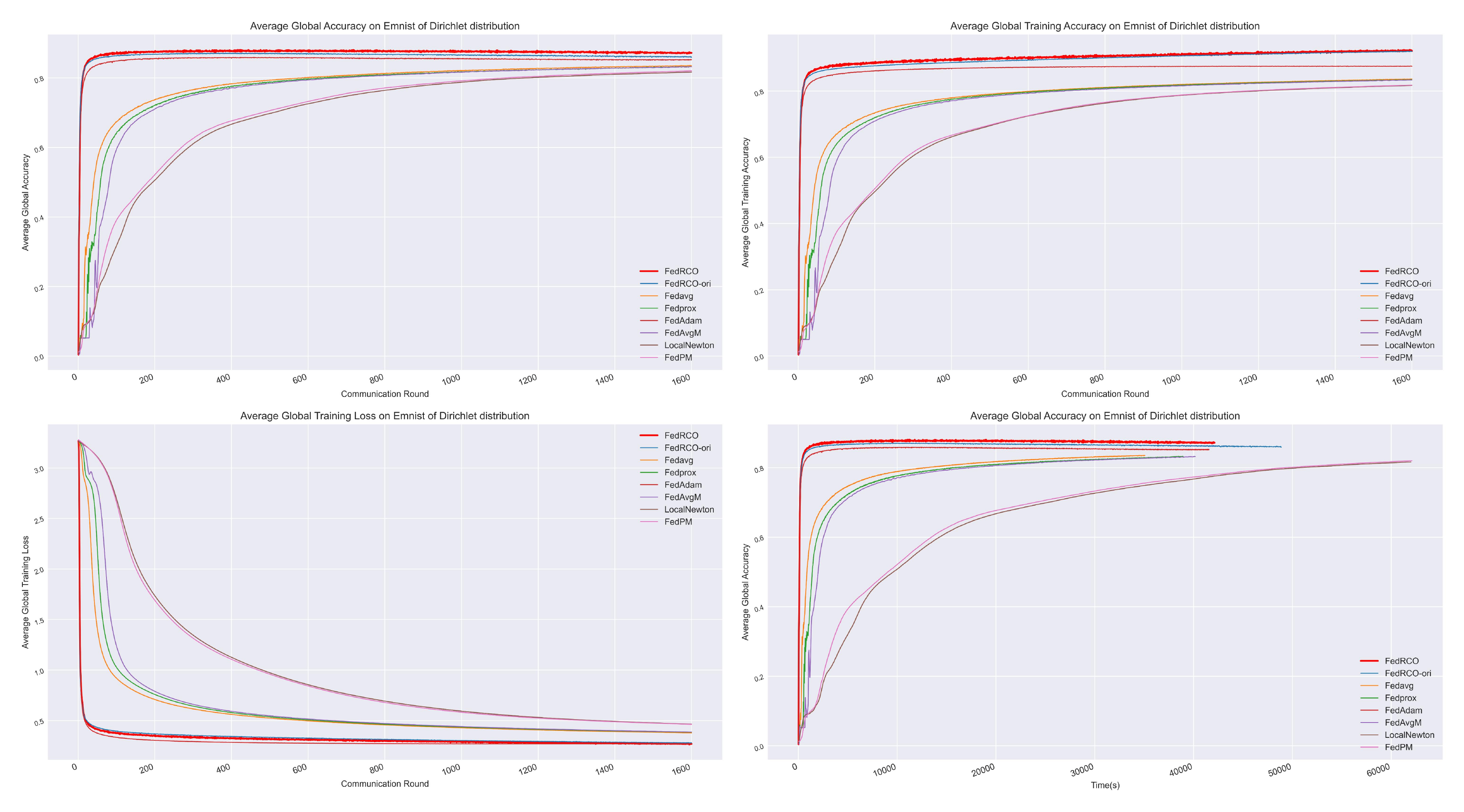}
  \caption{Dataset: EMNIST, Data distribution: Dirichlet 1, Party ratio: 0.8, Number clients: 100.}
  \label{appd.8}
\end{figure}

\begin{figure}[]
  \centering
  \includegraphics[width=6in]{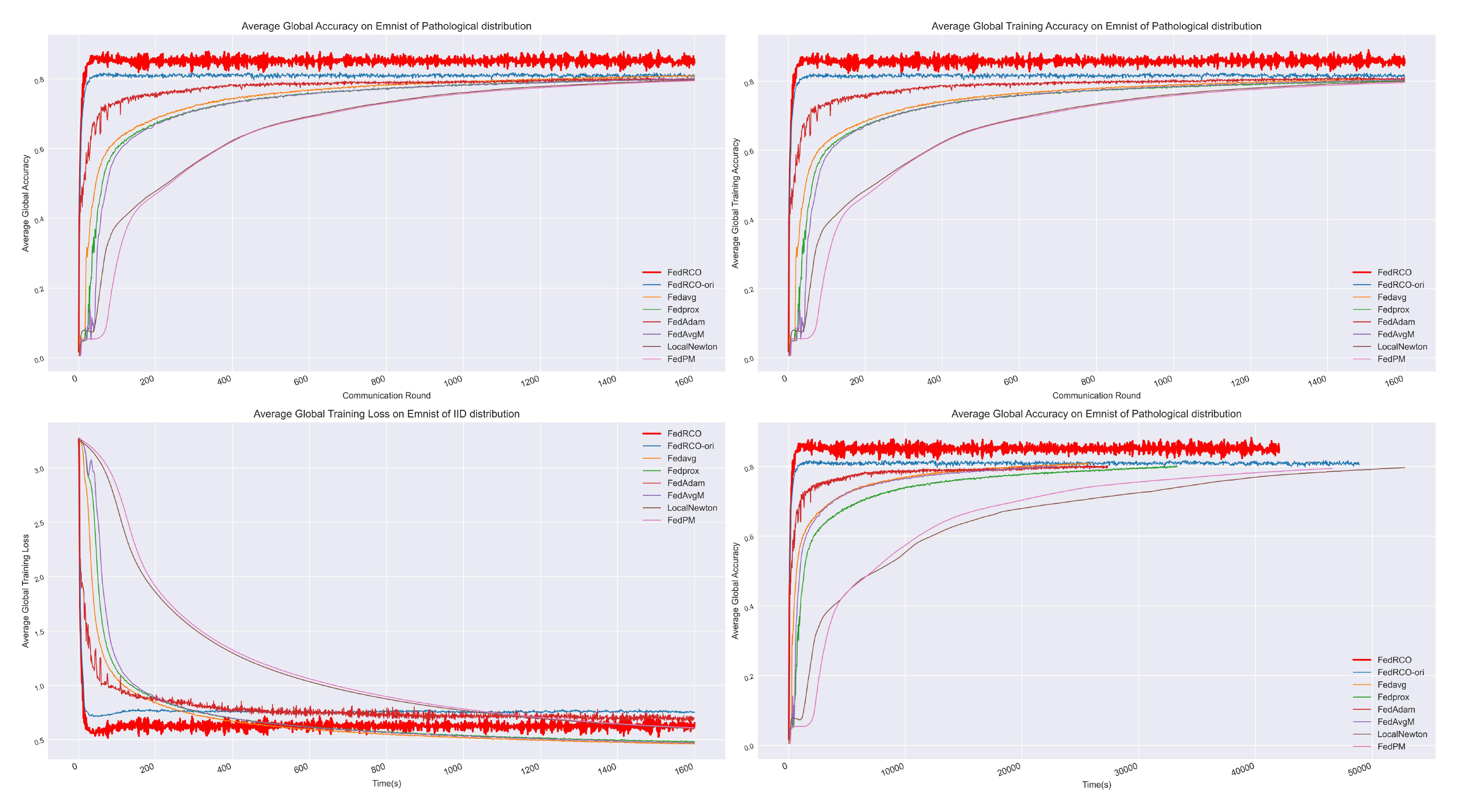}
  \caption{Dataset: EMNIST, Data distribution: Pathological 10, Party ratio: 0.8, Number clients: 100.}
  \label{appd.9}
\end{figure}

\begin{figure}[]
  \centering
  \includegraphics[width=6in]{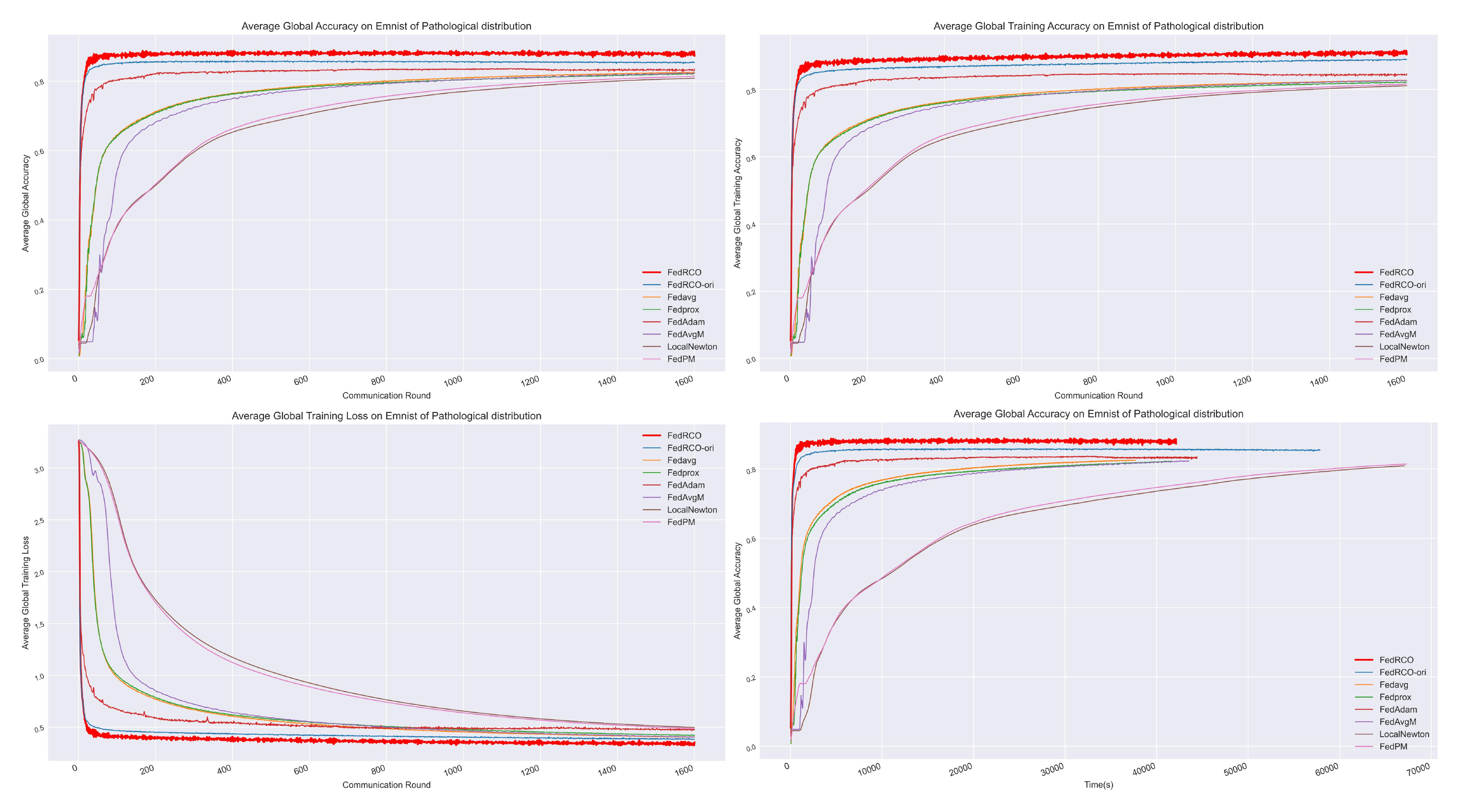}
  \caption{Dataset: EMNIST, Data distribution: Pathological 30, Party ratio: 0.8, Number clients: 100.}
  \label{appd.10}
\end{figure}

\begin{figure}[]
  \centering
  \includegraphics[width=6in]{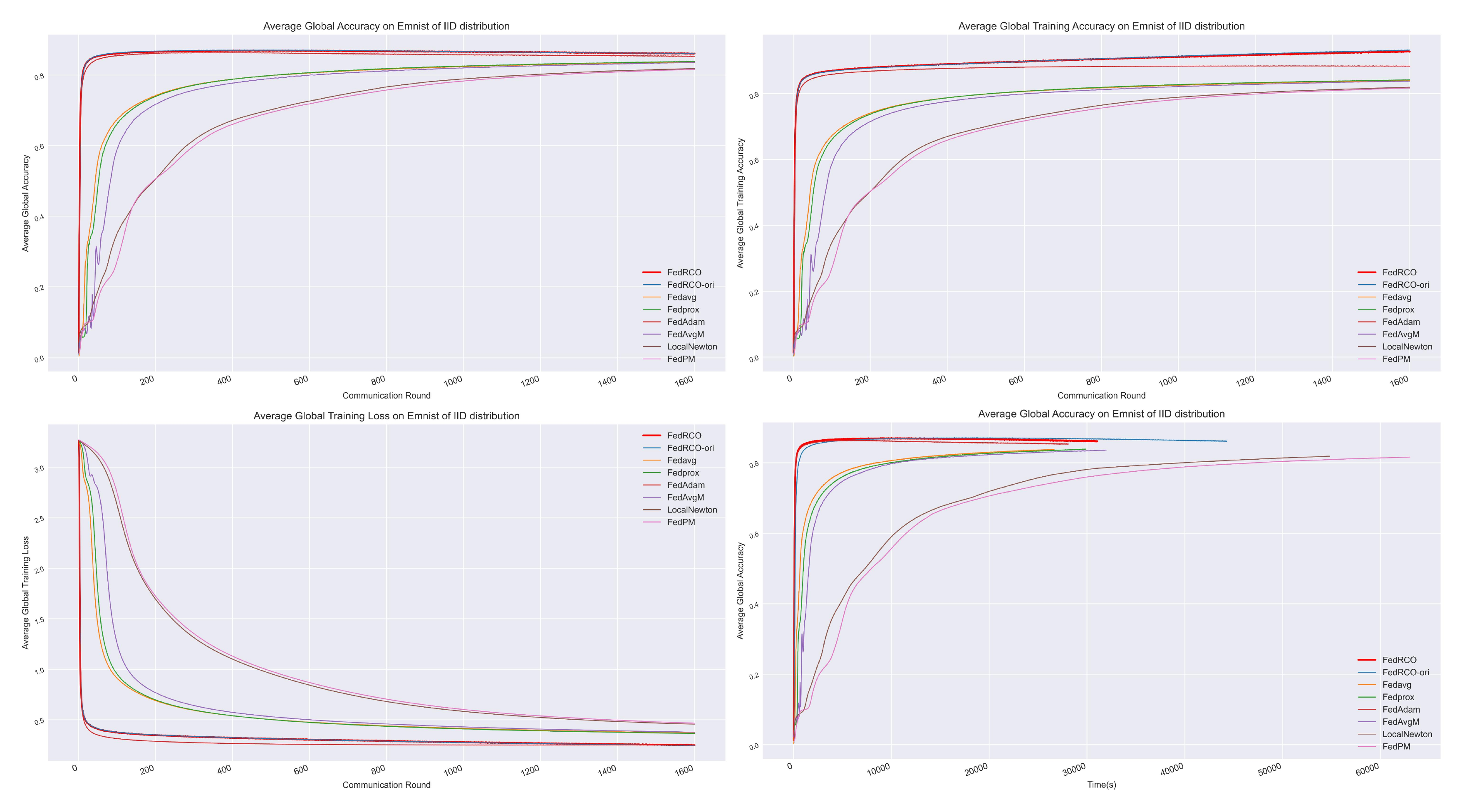}
  \caption{Dataset: EMNIST, Data distribution: iid, Party ratio: 0.8, Number clients: 100.}
  \label{appd.11}
\end{figure}

\begin{figure}[]
  \centering
  \includegraphics[width=6in]{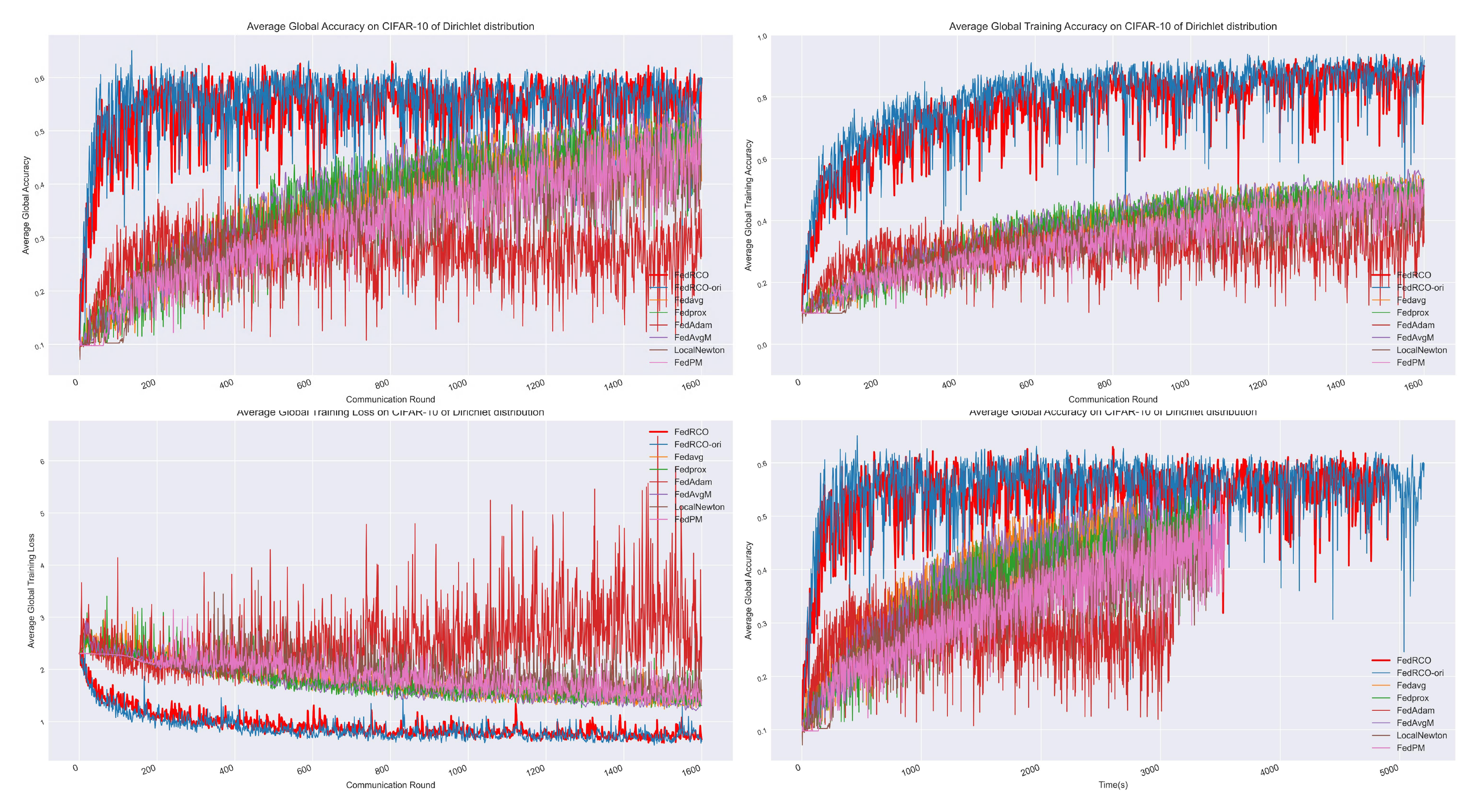}
  \caption{Dataset: CIFAR-10, Data distribution: Dirichlet 0.1, Party ratio: 0.1, Number clients: 100.}
  \label{appd.12}
\end{figure}

\begin{figure}[]
  \centering
  \includegraphics[width=6in]{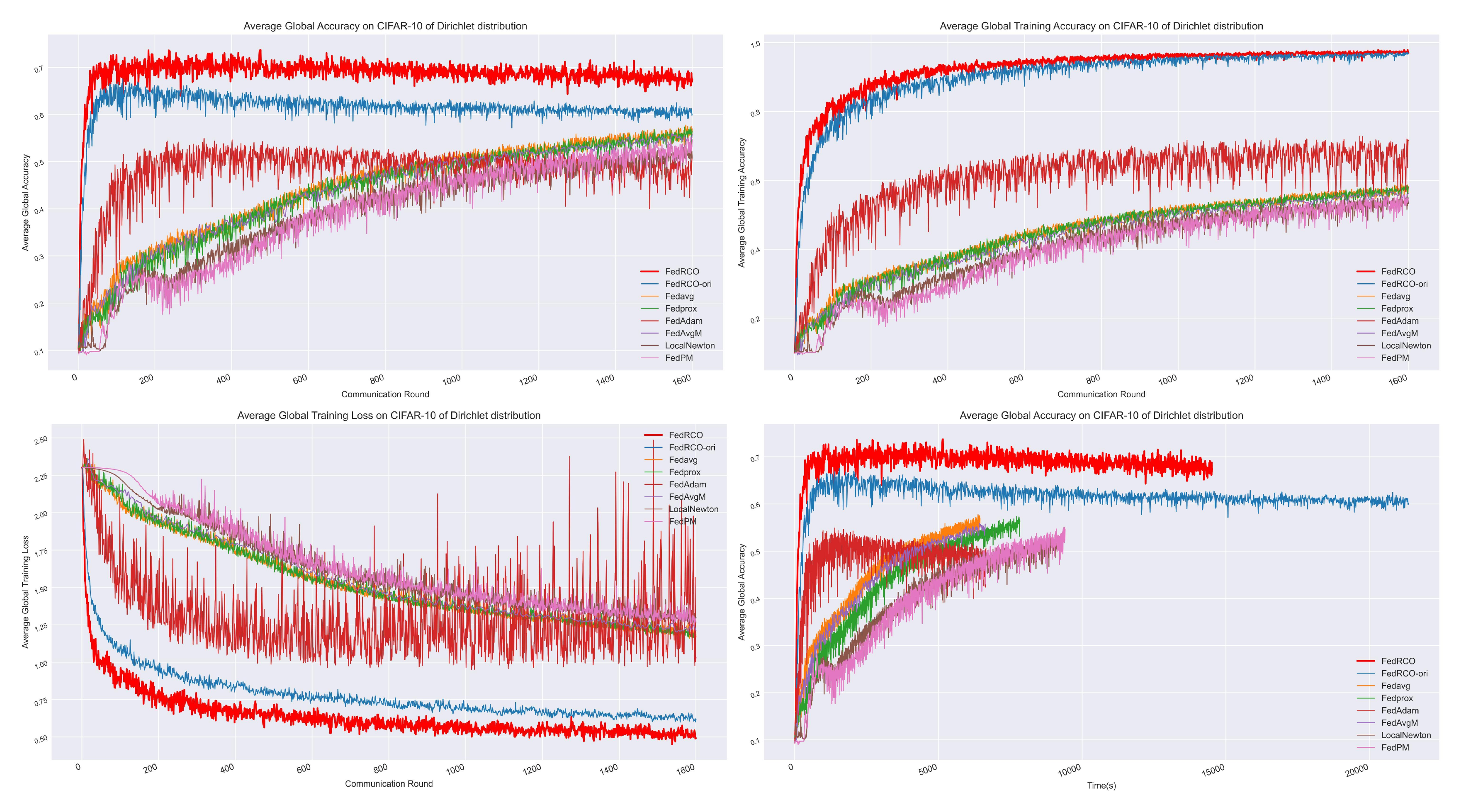}
  \caption{Dataset: CIFAR-10, Data distribution: Dirichlet 0.1, Party ratio: 0.5, Number clients: 100.}
  \label{appd.13}
\end{figure}

\begin{figure}[]
  \centering
  \includegraphics[width=6in]{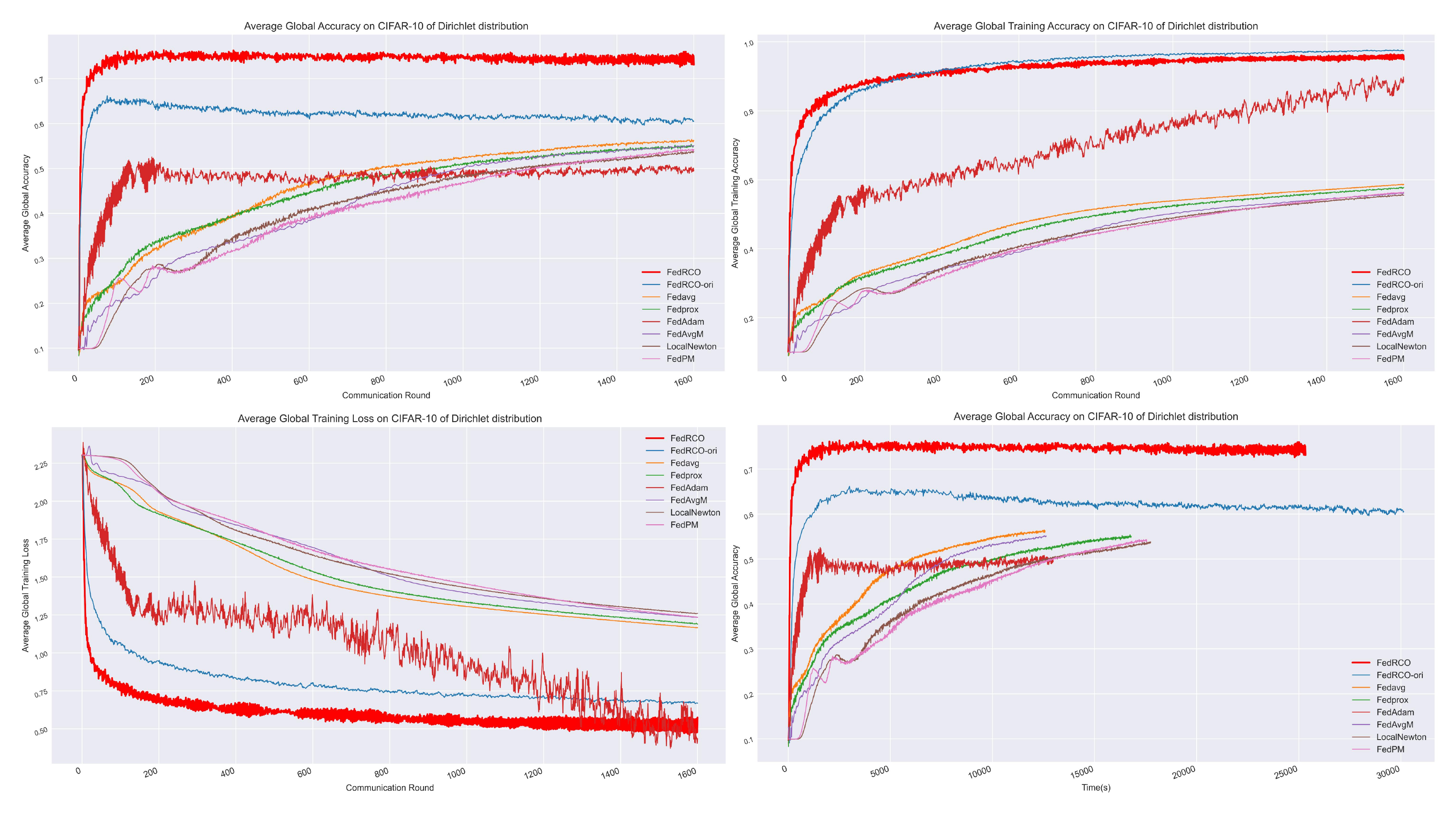}
  \caption{Dataset: CIFAR-10, Data distribution: Dirichlet 0.1, Party ratio: 1, Number clients: 100.}
  \label{appd.14}
\end{figure}

\begin{figure}[]
  \centering
  \includegraphics[width=6in]{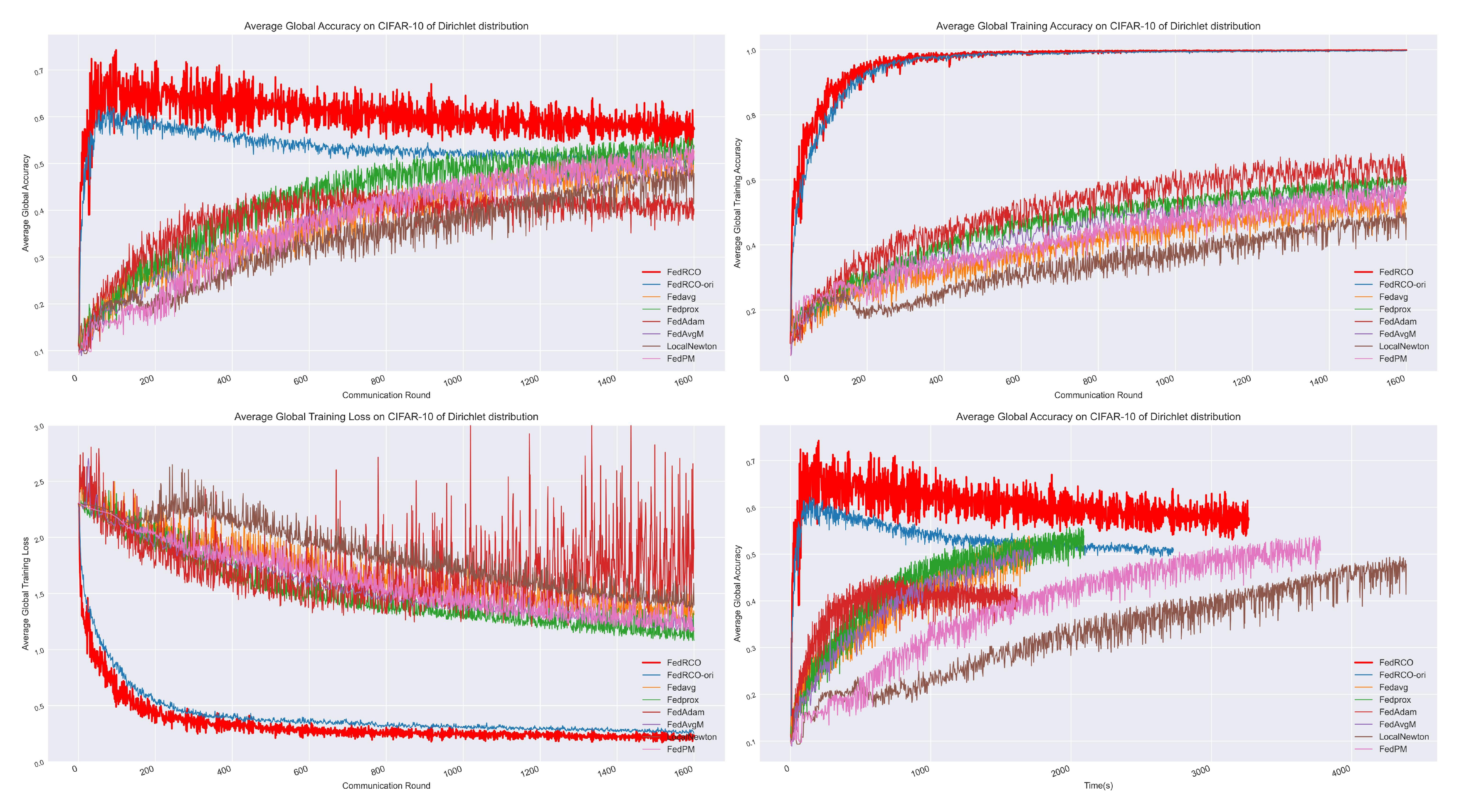}
  \caption{Dataset: CIFAR-10, Data distribution: Dirichlet 0.1, Party ratio: 0.8, Number clients: 10.}
  \label{appd.15}
\end{figure}

\begin{figure}[]
  \centering
  \includegraphics[width=6in]{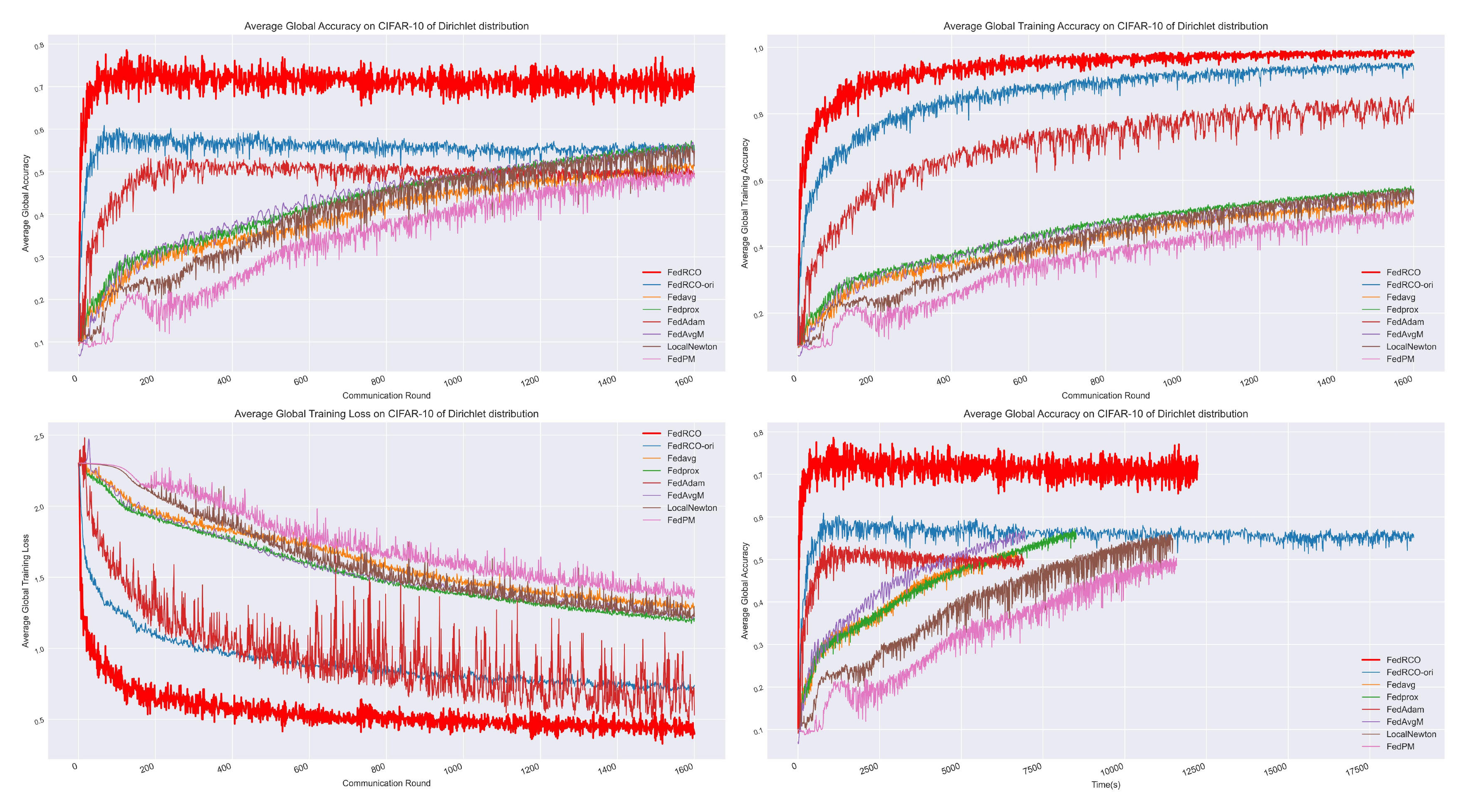}
  \caption{Dataset: CIFAR-10, Data distribution: Dirichlet 0.1, Party ratio: 0.8, Number clients: 50.}
  \label{appd.16}
\end{figure}

\begin{figure}[]
  \centering
  \includegraphics[width=6in]{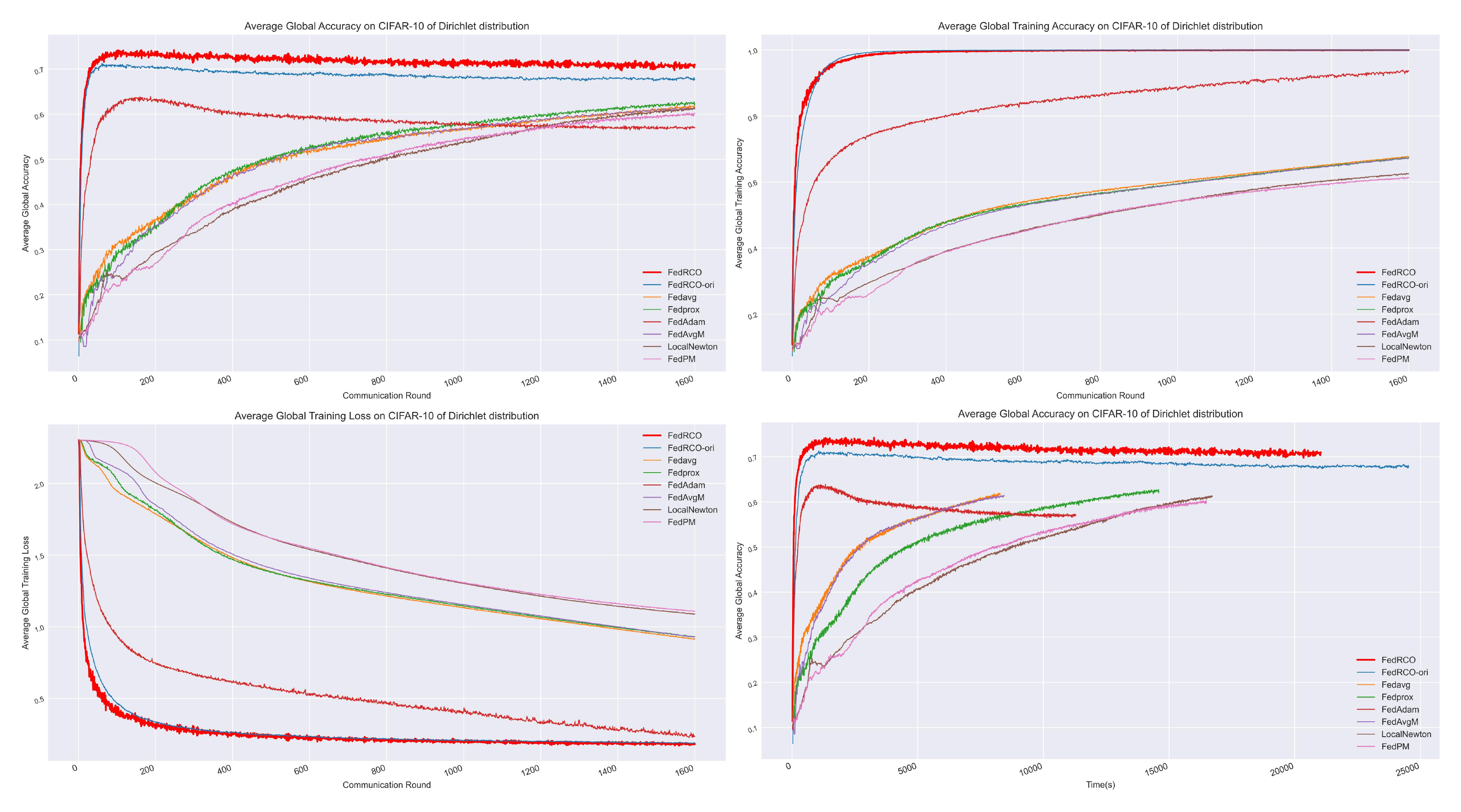}
  \caption{Dataset: CIFAR-10, Data distribution: Dirichlet 0.5, Party ratio: 0.8, Number clients: 100.}
  \label{appd.17}
\end{figure}

\begin{figure}[]
  \centering
  \includegraphics[width=6in]{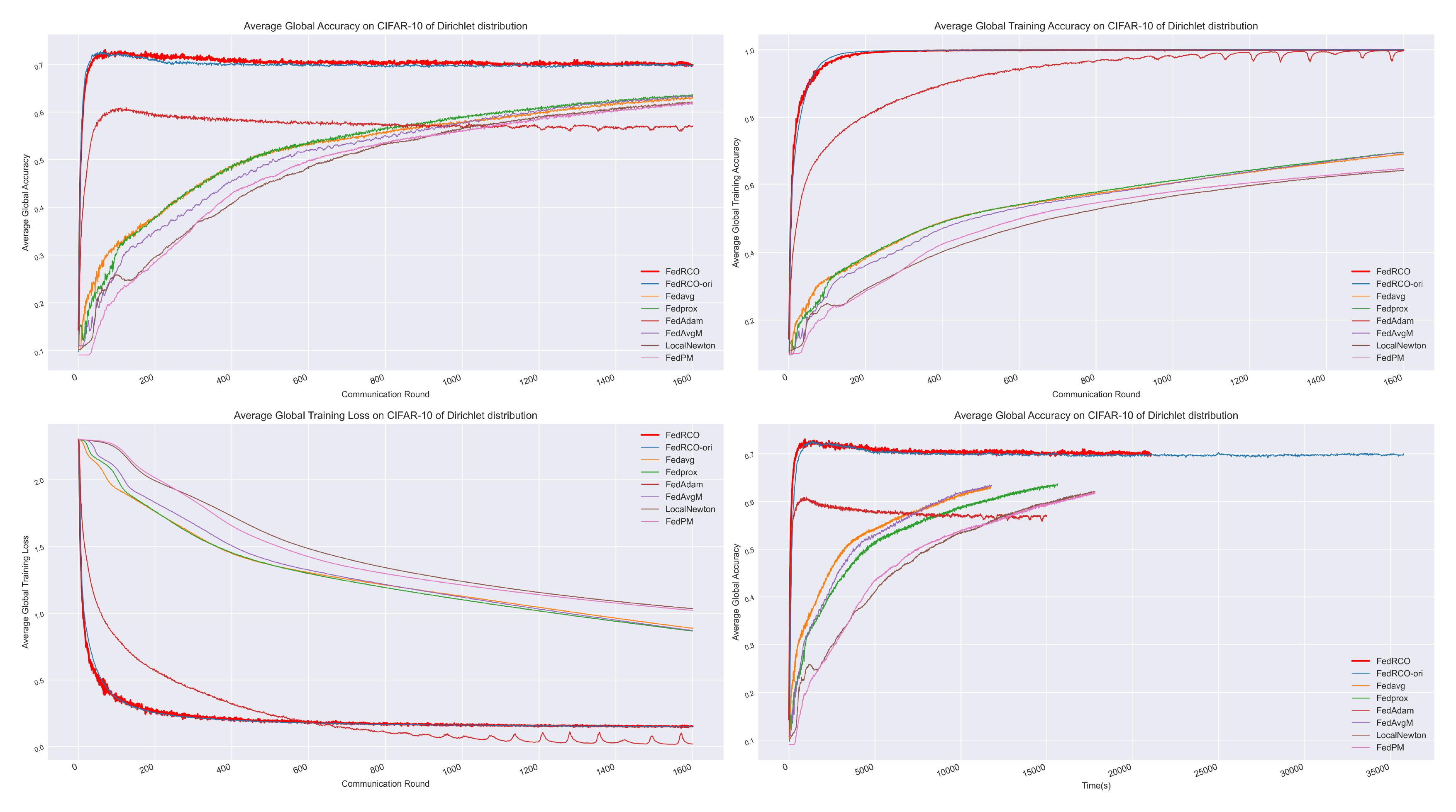}
  \caption{Dataset: CIFAR-10, Data distribution: Dirichlet 1, Party ratio: 0.8, Number clients: 100.}
  \label{appd.18}
\end{figure}

\begin{figure}[]
  \centering
  \includegraphics[width=6in]{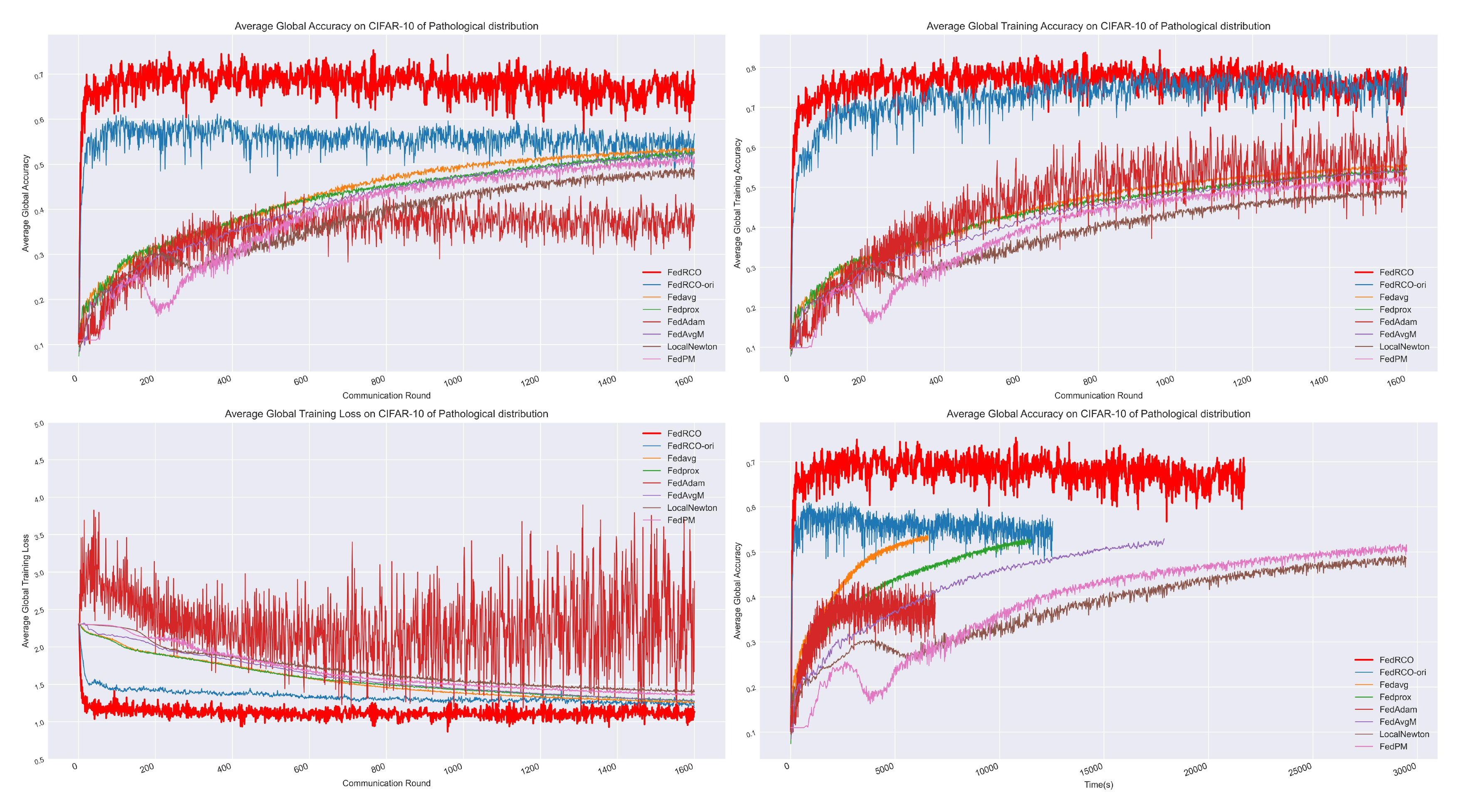}
  \caption{Dataset: CIFAR-10, Data distribution: Pathological 2, Party ratio: 0.8, Number clients: 100.}
  \label{appd.19}
\end{figure}

\begin{figure}[]
  \centering
  \includegraphics[width=6in]{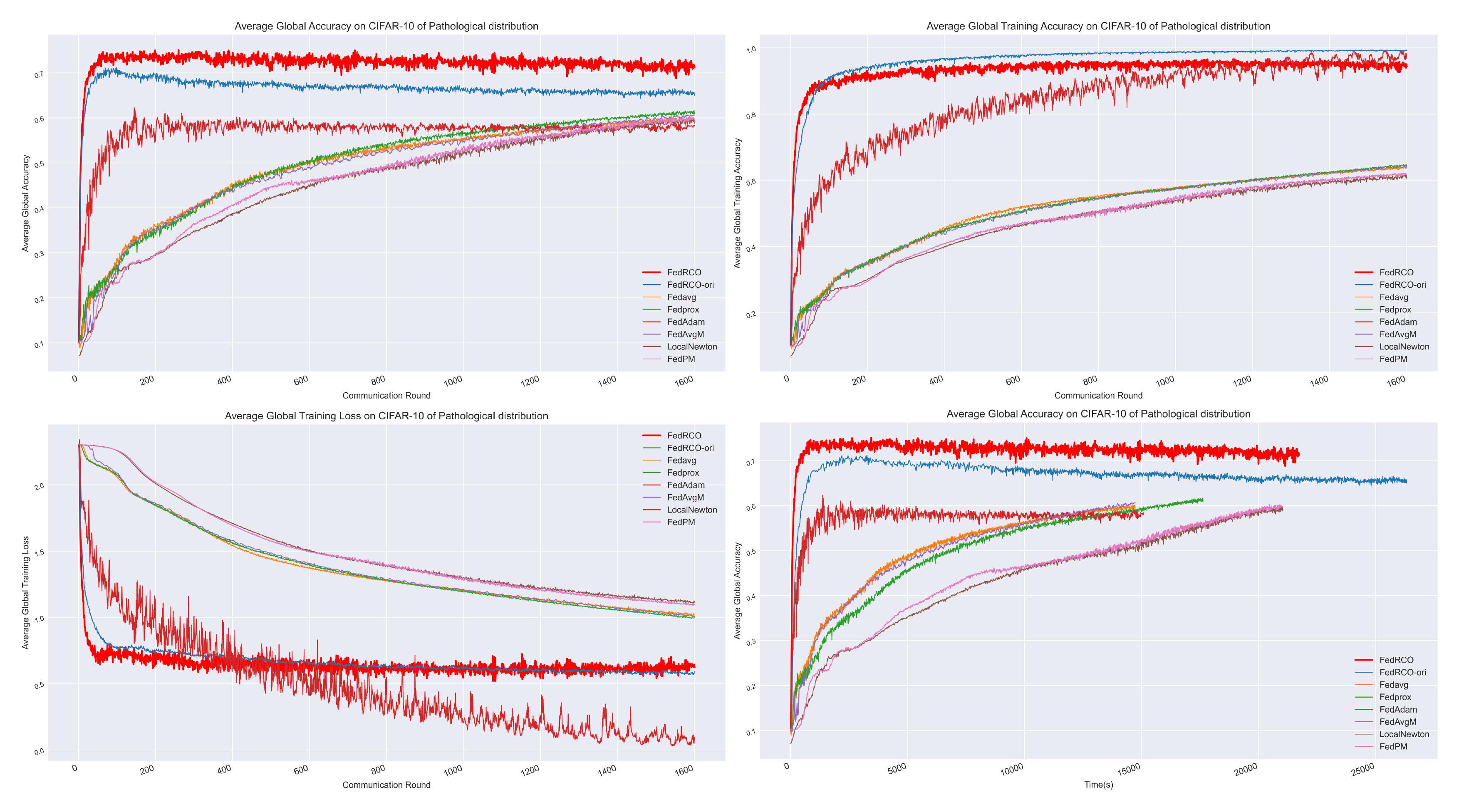}
  \caption{Dataset: CIFAR-10, Data distribution: Pathological 5, Party ratio: 0.8, Number clients: 100.}
  \label{appd.20}
\end{figure}

\begin{figure}[]
  \centering
  \includegraphics[width=6in]{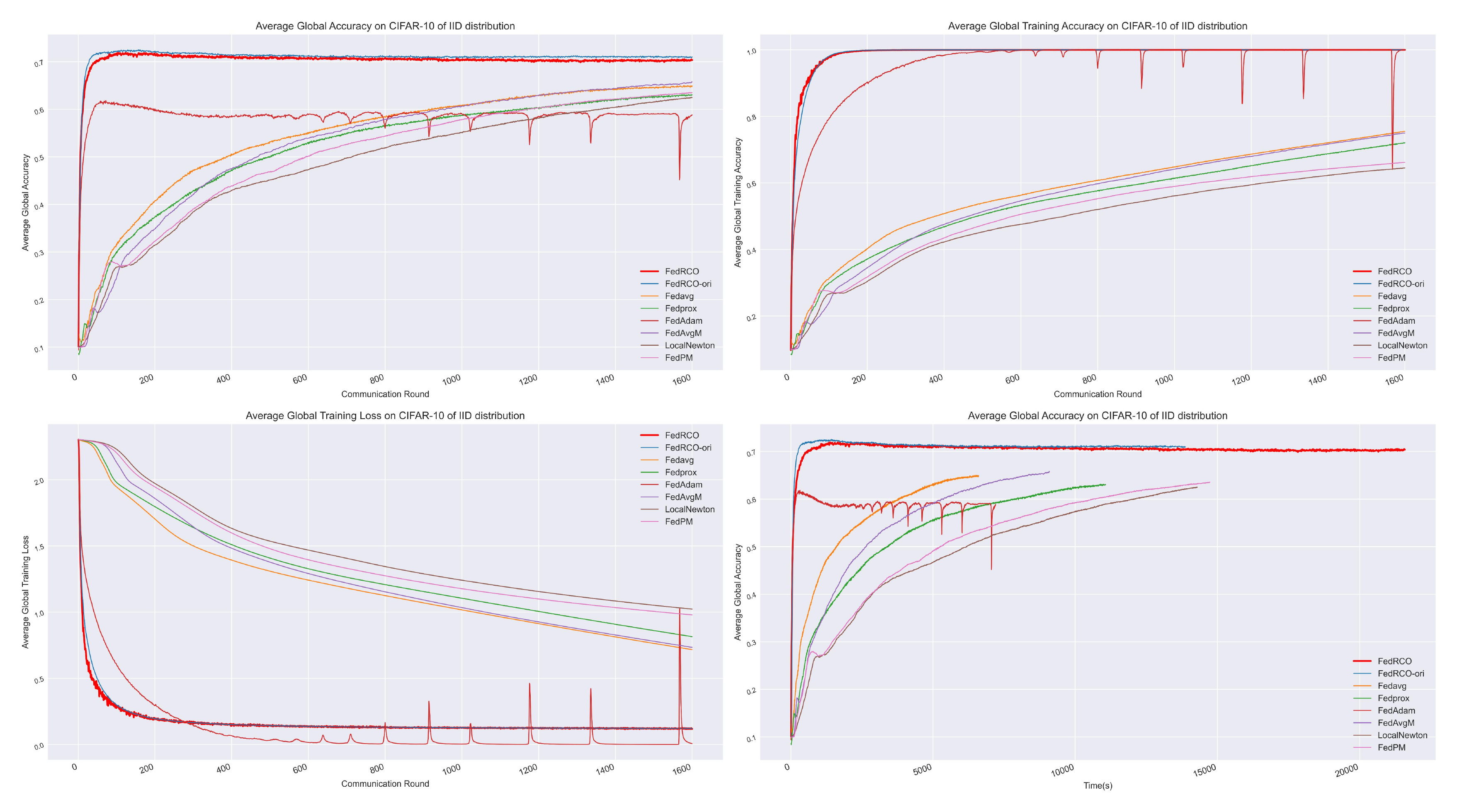}
  \caption{Dataset: CIFAR-10, Data distribution: iid, Party ratio: 0.8, Number clients: 100.}
  \label{appd.21}
\end{figure}


\end{document}